%% file: main.tex
\documentclass[10pt,twocolumn,letterpaper]{article}

\input{preamble}

\input{defs}

\begin{document}

\title{\name:\\Multi-Body Segmentation and Motion Estimation via 3D Scan Synchronization}

\author{Jiahui Huang$^1$\enskip He Wang$^2$\enskip Tolga Birdal$^2$\enskip Minhyuk Sung$^3$\enskip Federica Arrigoni$^4$\enskip \\ Shi-Min Hu$^1$\enskip Leonidas Guibas$^2$\\
$^1$Tsinghua University\enskip $^2$Stanford University \enskip $^3$KAIST \enskip $^4$ University of Trento}

\maketitle

\begin{abstract}
\vspace{-7pt}
\input{sec0_abstract.tex}
\end{abstract}

\input{sec1_intro.tex}
\input{sec2_related.tex}
\input{sec3_method.tex}
\input{sec4_experiment.tex}
\input{sec5_conclusion.tex}

{\small
\bibliographystyle{ieee_fullname}
\bibliography{main}
}

\clearpage
\setcounter{section}{0}
\newcommand\refpaper[1]{\cref{#1}}
\renewcommand{\thesection}{\Alph{section}}
\renewcommand{\thetable}{S\arabic{table}}
\renewcommand{\thefigure}{S\arabic{figure}}
\renewcommand{\theequation}{S.\arabic{equation}}
\title{\vspace{-2em}\name:\\Multi-Body Segmentation and Motion Estimation via 3D Scan Synchronization --- Supplementary Material\vspace{-3em}}
\author{}
\maketitle
\input{extra/sec_content}
\input{extra/sec_proof}
\input{extra/sec_detail}

\input{extra/sec_results}

\end{document}

%% file: preamble.tex
\usepackage{titling}
\usepackage{cvpr}
\usepackage[utf8]{inputenc}
\usepackage{times}
\usepackage{epsfig}
\usepackage{graphicx}
\usepackage{amsmath,mathtools}
\usepackage{amssymb}
\usepackage{amsthm}
\usepackage{bm}
\usepackage{enumitem}
\usepackage{mathtools}
\usepackage{makecell}
\usepackage[htt]{hyphenat}
\usepackage{fontawesome}
\usepackage{pifont}
\usepackage{color}
\usepackage{xcolor}
\usepackage{subfigure}
\usepackage[algo2e,inoutnumbered,linesnumbered,algoruled,vlined]{algorithm2e}
\usepackage{algpseudocode}
\usepackage{booktabs}
\usepackage{tabularx}
\usepackage{bm}
\usepackage{url}
\usepackage{bbm}
\usepackage{multirow}
\usepackage{systeme}
\usepackage{comment}
\usepackage{colortbl}
\usepackage{setspace}

\usepackage[pagebackref=true,breaklinks=true,letterpaper=true,colorlinks,bookmarks=false]{hyperref}

\newcommand{\name}{{MultiBodySync\xspace}}
\newcommand{\dataset}{DynLab\xspace}

\DeclareMathOperator*{\argmin}{argmin}

\newcolumntype{L}{>{\raggedright\arraybackslash}X}
\newcolumntype{C}{>{\centering\arraybackslash}X}

\cvprfinalcopy 

\def\cvprPaperID{2270} 


%% file: defs.tex
\newcommand{\x}{\bm{x}} 
\newcommand{\X}{\mathbf{X}} 
\newcommand{\Xset}{{\cal X}} 
\newcommand{\T}{\mathbf{T}} 
\newcommand{\Tset}{\mathcal{T}} 
\newcommand{\Tp}{\tilde{\mathbf{T}}} 
\newcommand{\Rp}{\tilde{\mathbf{R}}} 
\newcommand{\tp}{\tilde{\mathbf{t}}} 
\newcommand{\G}{\mathbf{G}} 
\newcommand{\Gset}{\mathcal{G}} 
\newcommand{\Gs}{\bm{g}} 
\newcommand{\Gw}{\tilde{\bm{g}}} 
\newcommand{\ps}{\bm{p}} 
\newcommand{\flow}{\mathbf{F}} 
\newcommand{\fpt}{\bm{f}} 
\newcommand{\Pm}{\mathbf{P}} 
\newcommand{\Pset}{\mathcal{P}} 
\newcommand{\Id}{\mathbf{I}} 
\newcommand{\R}{\mathbb{R}} 
\newcommand{\Rot}{\mathbf{R}} 
\newcommand{\tra}{\mathbf{t}} 
\newcommand{\Lap}{\mathbf{L}} 
\newcommand{\w}{\mathbf{w}} 
\newcommand{\Z}{\bm{\zeta}} 
\newcommand{\Znet}{\bm{\zeta}_\mathrm{net}} 
\newcommand{\Zm}{\mathbf{Z}} 
\newcommand{\Hw}{\mathbf{H}} 
\newcommand{\U}{\mathbf{U}} 
\newcommand{\K}{\mathbf{K}} 
\newcommand{\segnet}{\varphi_{\mathrm{mot}}} 
\newcommand{\denoisenet}{\varphi_{\mathrm{mlp}}} 
\newcommand{\flownet}{\varphi_{\mathrm{flow}}} 
\newcommand{\confnet}{\varphi_{\mathrm{conf}}} 
\newcommand{\Man}{\mathcal{M}} 
\newcommand{\one}{\mathbf{1}} 
\newcommand{\zero}{\mathbf{0}} 
\newcommand{\SEuc}{\mathrm{SE}(3)} 
\newcommand{\SOg}{\mathrm{SO}(3)} 
\newcommand{\D}{\mathbf{D}} 
\newcommand{\dvec}{\mathbf{d}} 
\newcommand{\E}{\mathbf{E}} 
\newcommand{\J}{\mathbf{J}} 
\newcommand{\gt}{\mathrm{gt}}   
\newcommand{\evals}{\bm{\Lambda} }
\newcommand{\ZW}{\tilde{\Zm} }
\newcommand{\evalw}{\tilde{\lambda} }
\newcommand{\UW}{\tilde{\U} }

\newcommand{\mydots}{...}

\newtheorem{thm}{Theorem}
\newtheorem{remark}{Remark}

\newtheorem{lemma}{Lemma}

\usepackage{cleveref}
\crefname{section}{\S}{\S\S}
\crefname{subsection}{\S}{\S\S}
\crefname{conj}{Conj.}{Conj.}

\Crefname{assumption}{\textbf{H}\hspace{-3pt}}{\textbf{H}\hspace{-3pt}}
\crefname{assumption}{\textbf{H}}{\textbf{H}}

\crefname{algorithm}{\text{Alg.}}{\text{Alg.}}
\crefname{assumption}{\textbf{H}}{\textbf{H}}
\crefname{equation}{\text{Eq}}{\text{Eq}}
\crefname{definition}{\text{Dfn.}}{\text{Dfn.}}
\crefname{lemma}{\text{Lemma}}{\text{Lemma}}
\crefname{dfn}{\text{Dfn.}}{\text{Dfn.}}
\crefname{thm}{\text{Thm.}}{\text{Thm.}}
\crefname{tab}{\text{Tab.}}{\text{Tab.}}
\crefname{fig}{\text{Fig.}}{\text{Fig.}}
\crefname{table}{\text{Tab.}}{\text{Tab.}}
\crefname{figure}{\text{Fig.}}{\text{Fig.}}
\crefname{section}{\text{Sec.}}{\text{Sec.}}

\newcommand{\parahead}[1]{\vspace{1.5mm}\noindent\textbf{#1.}\ }


%% file: sec0_abstract.tex
We present \name, a novel, end-to-end trainable multi-body motion segmentation and rigid registration framework for multiple input 3D point clouds.
The two non-trivial challenges posed by this multi-scan multi-body setting that we investigate are:
(i) guaranteeing correspondence and segmentation consistency across multiple input point clouds capturing different spatial arrangements of bodies or body parts; and
(ii) obtaining robust motion-based rigid body segmentation applicable to novel object categories.
We propose an approach to address these issues that incorporates spectral synchronization into an iterative deep declarative network, so as to simultaneously recover consistent correspondences as well as motion segmentation.
At the same time, by explicitly disentangling the correspondence and motion segmentation estimation modules, we achieve strong generalizability across different object categories.
Our extensive evaluations demonstrate that our method is effective on various datasets ranging from rigid parts in articulated objects to individually moving objects in a 3D scene, be it single-view or full point clouds.
Code at \url{https://github.com/huangjh-pub/multibody-sync}.

%% file: sec1_intro.tex
\vspace{-3pt}
\section{Introduction}
\vspace{-3pt}
\begin{figure}[t]
    \centering
    \includegraphics[width=\linewidth]{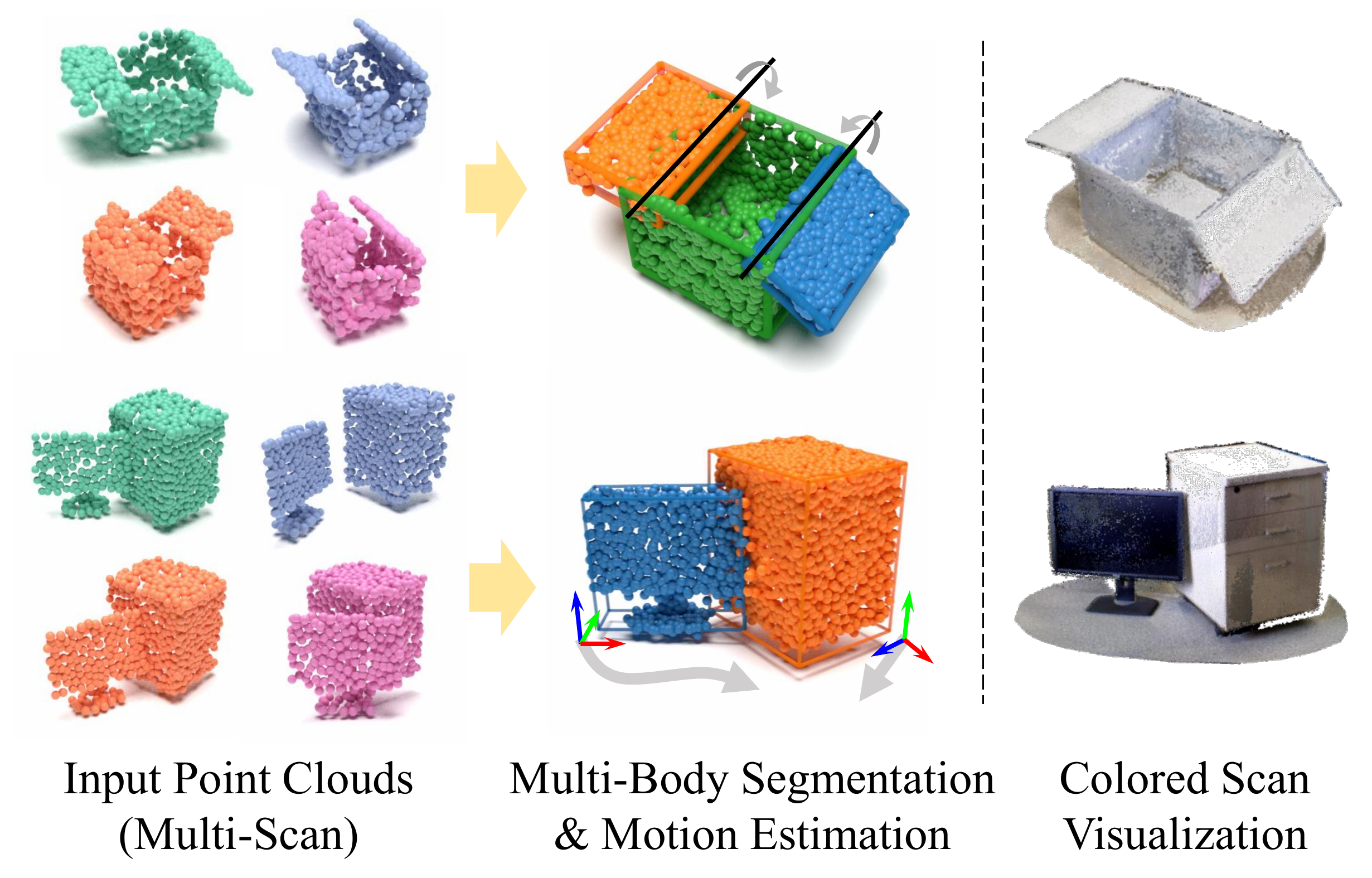}
    \caption{\name. Given an unordered set of point clouds, we simultaneously segment individual moving rigid parts/objects and register them. The transformed point clouds according to the first scan are aggregated and shown in the middle column. Note that the algorithm \emph{does not} use color information and the right column is shown just for visualization.\vspace{-3mm}}
    \label{fig:teaser}
\end{figure}

Motion analysis in dynamic point clouds is an emerging area, required by various applications such as surveillance, autonomous driving, and robotic manipulation.
Our human-made environments are dominated by \emph{rigid body} movements, ranging from articulated objects to solids like furniture or vehicles. 
These settings require us to address rigid motions of objects or object parts -- which is often referred to as the \emph{multi-body} motion estimation problem.
Despite its importance, previous work has mainly focused on specific scenarios with known category semantics, like category-level articulated object segmentation~\cite{li2020category}, indoor scene instance relocalization~\cite{wald2019rio}, or car movement detection~\cite{wu2020motionnet}, leaving the literature of generic motion segmentation relatively unexplored.

Different from traditional single scan analysis algorithms like semantic segmentation~\cite{landrieu2018large}, the most challenging part in multi-body motion analysis is to disambiguate and distinguish rigid bodies. There, we are naturally required to jointly process and relate \emph{multiple} inputs, to effectively find \emph{consistent} motion-based part/object segmentations as well as point correspondences to enable a multi-way registration.
It is even more challenging when the capture is not temporally dense, \ie,  an intermittent acquisition that does not follow a stream such as a video, and might contain large pose variations, hampering naive temporal tracking.

In this paper, we introduce a multi-scan multi-body segmentation and motion estimation problem, where the goal is to simultaneously discover and register rigid bodies from multiple scans, represented either as full or partial point clouds, where objects come from unseen categories.
As an effective solution, we present \name, a fully end-to-end trainable deep declarative architecture~\cite{Gould2019} able to process an arbitrary number of unordered point sets.
As shown in~\cref{fig:teaser}, given a set of scans, \name~begins relating pairs of scans via \emph{3D scene flow}~\cite{yi2018deep,vedula1999three} and confidence estimation. 
Then, the following two differentiable (permutation and segmentation) synchronization modules, which are central to our approach, respectively enforce the consistency of pairwise point correspondences and motion segmentation labelings across different scans.
Our design explicitly decouples geometry and motion, making \name~generalizable to \emph{unseen categories} without sacrificing robustness. 

We evaluate \name~on various datasets composed of full synthetic point clouds and partial real scans with articulated and solid objects. 
We also contribute a new dataset~\dataset with 8 scenes and 64 scan fragments of distinctly moving objects. 
Our extensive evaluations demonstrate that our algorithm outperforms the state-of-the-art by a large margin on both multi-body motion segmentation and motion estimation. 
In brief, our contributions are:
\vspace{-3pt}
\begin{enumerate}[noitemsep,leftmargin=\parindent]
\item We introduce a novel end-to-end trainable architecture for solving the multi-scan multi-body motion estimation and segmentation problem.
\item We theoretically analyze the spectral characteristics of the proposed weighted permutation synchronization.
\item To the best of our knowledge, we showcase the first cross-category generalization for the task at hand on both synthetic and real datasets, for both articulated part-level and object-level regimes.
\end{enumerate}

%% file: sec2_related.tex
\section{Related Works}
\vspace{-3pt}

\parahead{Dynamic scene understanding}
The modeling of 3D dynamic scenes in deep learning literature is often formulated as a 4D data analysis, as done in seminal works like \cite{liu2019meteornet,choy20194d}. Ability to infer spatiotemporal geometric properties has recently motivated research in \emph{3D scene flow} as a form of low-level dynamic scene representation~\cite{liu2019flownet3d,tishchenko2020selfflow,wang2019flownet3d++,puy20flot,niemeyer2019occupancy,rempe2020caspr,ma2019deep}. 
Domain-specific knowledge can be employed to give better predictions as done in autonomous driving~\cite{hu2019joint,behl2019pointflownet,wu2020motionnet} or articulated object analysis~\cite{yan2020rpm,wang2019shape2motion}. The most recent dynamic SLAM works~\cite{huang2020clustervo,bescos2020dynaslam,zhang2020vdo,xu2019mid} also rely heavily on semantic cues.
While some works~\cite{niemeyer2019occupancy,rempe2020caspr} advocates continuous temporal-dynamics modeling, we instead assume discrete non-sequential input and enforce consistency using synchronization.
Similarly, \cite{halber2019rescan,wald2019rio} propose to perform instance-level re-localization in a changed scene. 
Nevertheless, we do not assume a pre-segmentation of the scene, but instead perform joint motion segmentation.

\parahead{Multi-body motion}
Provided point correspondences between two point clouds/images, rigid-body motion segmentation becomes a multi-model fitting problem, amenable for factorization techniques~\cite{costeira1998multibody,li2007projective,xu20193d}, clustering~\cite{huang2019clusterslam}, graph optimization~\cite{magri2019fitting,isack2012energy,birdal2017cad} or deep learning~\cite{kluger2020consac}.
Among others, \cite{yi2018deep} handles raw scans and segments the rigidly moving parts using a Recurrent Neural Network (RNN).
\cite{hayden2020nonparametric} fits non-parametric part models to sequential 3D data without needing explicit correspondences.
However, to our best knowledge, no prior work can handle multiple scans while enforcing multi-way consistency like we do.

\parahead{Synchronization} 
The art of consistently recovering absolute quantities from a collection of ratios is now a basic component of the classical multi-view/shape analysis pipelines~\cite{salas2013slam++,cadena2016past,carlone2015initialization}. Various aspects of the problem have been vastly studied: different group structures~\cite{govindu2014averaging,govindu2004lie,birdal2019probabilistic,arrigoni2017synchronization,Arrigoni2019,huang2019tensor,hartley2013rotation,Arrigoni2019,wang2013exact,chaudhury2015global,thunberg2017distributed,tron2014distributed,arrigoni2016spectral,bernard2015solution}, closed-form solutions~\cite{arrigoni2016spectral,arrigoni2017synchronization,Arrigoni2019}, robustness~\cite{chatterjee2017robust}, certifiability~\cite{rosen2019se}, global optimality~\cite{briales2017cartan}, learning-to-synchronize~\cite{huang2019learning,purkait2019neurora,gojcic2020learning} and uncertainty quantification~\cite{tron2014statistical,birdalSimsekli2018,birdal2020synchronizing,birdal2019probabilistic}. 
In this work, we are concerned with synchronizing correspondence sets, otherwise known as \emph{permutation synchronization} (PS)~\cite{pachauri2013solving} and motion segmentations~\cite{arrigoni2019motion}. PS is rich in the variety of algorithms: low-rank formulations~\cite{yu2016globally,wang2018multi}, convex programming~\cite{hu2018distributable}, distributed optimization~\cite{hu2018distributable}, multi-graph matching\cite{schiavinato2017synchronization} or Riemannian optimization~\cite{birdal2019probabilistic}. 
Out of all those, we are interested in the spectral methods of \cite{arrigoni2017synchronization,maset2017practical} as they provide efficient, closed-form solutions deployable within a \emph{deep declarative network}~\cite{Gould2019} like ours. 

To the best of our knowledge, synchronization of correspondences~\cite{maset2017practical} or {motion segmentation}~\cite{arrigoni2019motion} have not been explored in the context of deep learning. This is what we do in this paper to tackle the consistent multi-body motion estimation and segmentation.

%% file: sec3_method.tex
\section{Method}
\vspace{-3pt}
\begin{figure*}[t]
    \centering
    \includegraphics[width=\linewidth]{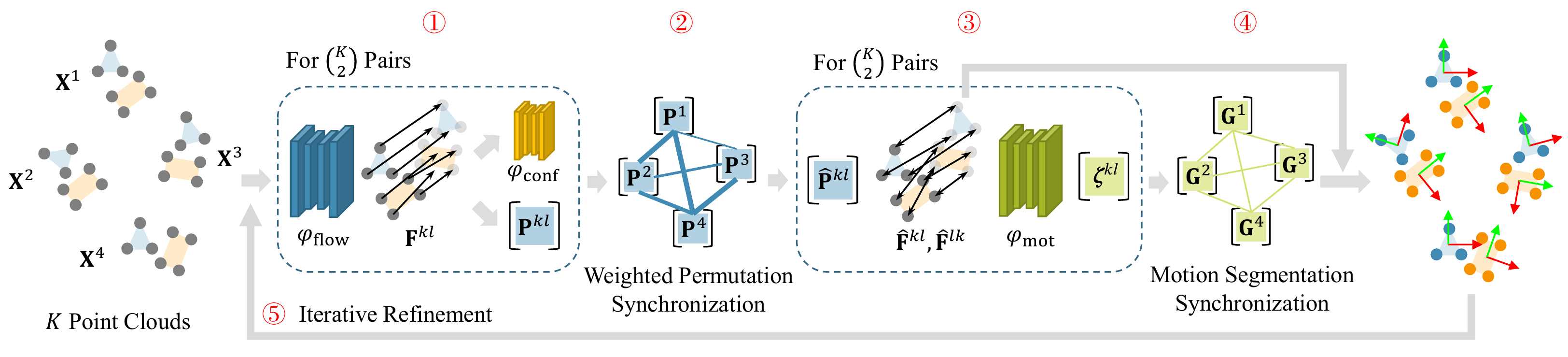}
    \caption{Our pipeline. \ding{172} Given multiple input point clouds, we first estimate pairwise scene flows. \ding{173} The point correspondences (permutations) computed from the flows are then synchronized in a weighted fashion to enforce consistencies. \ding{174} Pairwise relative segmentations are subsequently estimated from the flows, and \ding{175} further synchronized to get absolute motion segmentation. The pose of each part can be recovered by a weighted Kabsch algorithm. Our pipeline is fully differentiable and can be iterated (\ding{176}) to get improved results.\vspace{-4mm}}
    \label{fig:pipeline}
\end{figure*}
\parahead{Problem setting and notation}
Suppose we observe a set of $K$ point clouds $\Xset= \{ \X^k \in \R^{3 \times N}, k\in[1,K] \}$ where each point cloud $\X^k= \begin{bmatrix} \x^k_1, \mydots, \x^k_i, \mydots, \x^k_N \end{bmatrix}$ contains $N$ points in $\R^3$ and sampled from the same object with $S$ \emph{independently moving} rigid parts indexed by $s$.
Each point is assumed to belong to one of the $S$ rigid parts and we denote the binary \emph{point-part association matrices} as $\Gset = \{ \G^k \in [0,1]^{N \times S} \}$ where $G^{k}_{is}=1$ if $\x^k_i$ belongs to the $s^\mathrm{th}$ rigid part and $G^{k}_{is}=0$ otherwise\footnote{Throughout our paper we use superscript $k,l$ to index point-clouds, subscript $i,j$ to index points and subscript $s$ to index rigid parts.}. 
The rigid motions for each part $s$ in each point cloud $k$ is defined as $\Tset = \{ \T^k_s \in \SEuc, k\in[1,K], s\in[1,S] \}$, with the rotational part being $\Rot^k_s\in\SOg$ and the translational part being $\tra^k_s\in\R^3$.
Our final goal is to infer $\Gset$ and $\Tset$ given $\Xset$. 

\parahead{Summary}
The core of our approach is a fully differentiable deep network fusing rigid dynamic information from multiple 3D scans as outlined in ~\cref{fig:pipeline}.
We begin by explicitly predicting pairwise soft correspondences across {all pairs of point clouds} while enforcing consistency via a \emph{weighted permutation synchronization} (\cref{subsec:flow}).
Next, the point clouds are segmented using a novel motion-based segmentation network and also further synchronized by a subsequent \emph{motion segmentation synchronization} module (\cref{subsec:group}).
Finally, the correspondences and segmentations are used to recover the 6-DoF transformation for each of the individual rigid parts. The whole procedure can be iterated to refine the results.
The pipeline can be readily trained end-to-end and we describe our training procedure in \cref{subsec:train}.

\subsection{Flow Estimation and Synchronization}
\label{subsec:flow}
\vspace{-3pt}
Our approach starts with point correspondence estimation between all $\binom{K}{2}$ pairs of point clouds. 
We tackle this problem by predicting a 3D scene flow $\flow^{kl} = \begin{bmatrix} \fpt^{kl}_1, \mydots, \fpt^{kl}_i, \mydots, \fpt^{kl}_N \end{bmatrix} \in \R^{3 \times N}$ for each {point cloud pair indexed by} $(k,l)$ using a deep neural network $\flownet(\cdot)$, \ie $\flow^{kl}=\flownet(\X^k,\X^l)$, so that $\X^k + \flow^{kl} \stackrel{P}{=} \X^l$ holds up to a permutation.
The architecture of $\flownet$ inspired by \emph{Point PWC-Net}~\cite{wu2019pointpwc} is detailed in the supplementary material.

Flow signals, estimated in a pairwise fashion, are not informed about the multiview configuration at our disposal. 
To ensure multi-way consistent flows, we employ the weighted variant of permutation synchronization~\cite{maset2017practical} inspired by~\cite{gojcic2020learning,huang2019learning} where a closed-form solution is given under spectral relaxation. 
We begin by the observation that any flow $\flow^{kl}$ would induce a soft assignment matrix $\Pm^{kl} \in \Man^{N \times N}$ based on the nearest-neighbor distances:
\begin{equation}
    P^{kl}_{ij} = \frac{\exp{(\delta^{kl}_{ij}/\tau)}}{\sum_{\mathrm{j}=1}^N \exp{(\delta^{kl}_{i\mathrm{j}}/\tau)}},\,\,
    \delta^{kl}_{ij}=\lVert \x^k_i + \fpt^{kl}_i - \x^l_j \rVert^2
\end{equation}
where $\tau$ is the temperature of the \emph{softmax}.
The \emph{multinomial manifold} $\Man$ of row-stochastic matrices is a \emph{continuous relaxation} of the (partial) permutation group $\Pset$.

\parahead{Outlier filtering} 
To take into account the noise, missing points, or errors in the network, we further associate a confidence value $c^{kl}_i\in\R$ to each point $\x_i^k$ and its corresponding flow vector $\fpt_i^{kl}$ through another network $\confnet(\cdot):\R^{7\times N}\mapsto\R^N$ inspired from OANet~\cite{zhang2019learning}.
The input to this network are the tuples $\{(\x_i^k, \x_i^k + \fpt_i^{kl}, \argmin_j \delta_{ij}^{kl} ) \in \R^7\}_{i=1}^N$ and we provide the architectural details in the supplementary. 
The last dimension of this tuple measures the quality of the flow vector via the distance between the transformed points and their nearest neighbors, thereby detecting spurious flow predictions. 
The final $w^{kl}$ in~\cref{eq:wgcl} reflects the overall quality of the corresponding $\Pm^{kl}$. 
Here we choose $w^{kl}$ as the average confidence of all points, \ie, $w^{kl} = \sum_{i=1}^N c^{kl}_i / N$. 

\parahead{Consistent correspondences} 
We now use the predictions $\{\Pm^{kl}, w^{kl}\}_{(k,l)}$ to achieve multiview consistent assignments. 
To this end, we deploy a \emph{differentiable synchronization} algorithm inspired by~\cite{maset2017practical}.
We first introduce \emph{absolute} permutation matrices $\Pm^k$ which map each point in $\X^k$ to a \emph{universe} space and stack them as $\ps = [\dots,(\Pm^k)^\top,\dots]^\top$. We solve for the best $\ps$ minimizing:
\begin{equation}
\label{eq:sync}
E(\ps) = \sum_{k=1}^K \sum_{l=1}^K w^{kl} \lVert \Pm^k - \Pm^{kl} \Pm^l \rVert^2_F.
\end{equation}

\begin{thm}[Weighted synchronization]
The spectral solution to the \emph{weighted synchronization problem} in~\cref{eq:sync} $\ps$ is given by the $N$ eigenvectors of $\Lap$ corresponding to the smallest $N$ eigenvalues, where $\Lap \in \R^{KN \times KN}$ is the \textbf{weighted} Graph Connection Laplacian (GCL) constructed by tiling all $\Pm^{kl}$ matrices weighted by the related $w^{kl}$:
\begin{align}
\label{eq:wgcl}
\hspace{-3mm}\Lap = \begin{bmatrix}
w^1 \Id_N & -w^{12} \Pm^{12} & \dots & -w^{1K} \Pm^{1K} \\
-w^{21} \Pm^{21} & w^2 \Id_N & \dots & -w^{2K} \Pm^{2K} \\
\vdots & \vdots & \ddots & \vdots \\
-w^{K1} \Pm^{K1} & -w^{K2} \Pm^{K2} & \dots & w^K \Id_N,
\end{bmatrix},
\end{align}
with $w^k = \sum_{l=1,l\neq k}^K w^{kl}$ and $\Id_N\in\R^{N\times N}$ is the identity.
\end{thm}
\begin{proof}
Please refer to the supplementary material.
\end{proof}

This \emph{spectral solution} requires only an eigen-decomposition lending itself to easy differentiation~\cite{huang2019learning,gojcic2020learning}.
The synchronized soft correspondence $\hat{\Pm}^{kl}$ is then extracted as the $(k,l)$-th $N\times N$ block of $\ps \ps^\top$. 
As a consequence of the relaxation, we cannot ensure that each sub-matrix of $\ps \ps^\top$ would be a valid permutation. 
To preserve differentiability we avoid Hungarian-like projection operators~\cite{maset2017practical} and propose to directly compute the induced flow $\hat{\flow}^{kl}=\begin{bmatrix} \mydots, \hat{\fpt}^{kl}_i, \mydots \end{bmatrix}$ using a softmax normalization on the synchronized soft correspondences:
\begin{equation}
    \hat{\fpt}^{kl}_i = \frac{\sum_{j=1}^N \bar{P}^{kl}_{ij} (\x_{j}^l - \x_{i}^k)}{\sum_{j=1}^N \bar{P}^{kl}_{ij}}, \quad
    \bar{P}^{kl}_{ij}=\exp{(\hat{P}^{kl}_{ij}/t)}.
\end{equation}
Intuitively, this amounts to using the normalized synchronized result as a soft-assignment matrix, diminishing the effect of non-corresponding matches (false positives).

\subsection{Motion Segmentation}
\label{subsec:group}
\vspace{-3pt}
Based upon the multiview consistent flow output $\hat{\flow}^{kl}$, we now predict the point-part associations $\Gset$.
Since we are not provided with consistent labeling of the parts, instead of predicting $\G^k$ directly, we estimate for all $\binom{K}{2}$ point cloud pairs a relative \emph{motion segmentation matrix} $\Z^{kl} \in [0,1]^{N \times N}$, where $\zeta^{kl}_{ij}$ is 1 when $\x_i^k$ and $\x_j^l$ belong to the same rigid body, and $0$ otherwise.

\begin{figure}[t]
    \centering
    \includegraphics[width=\linewidth]{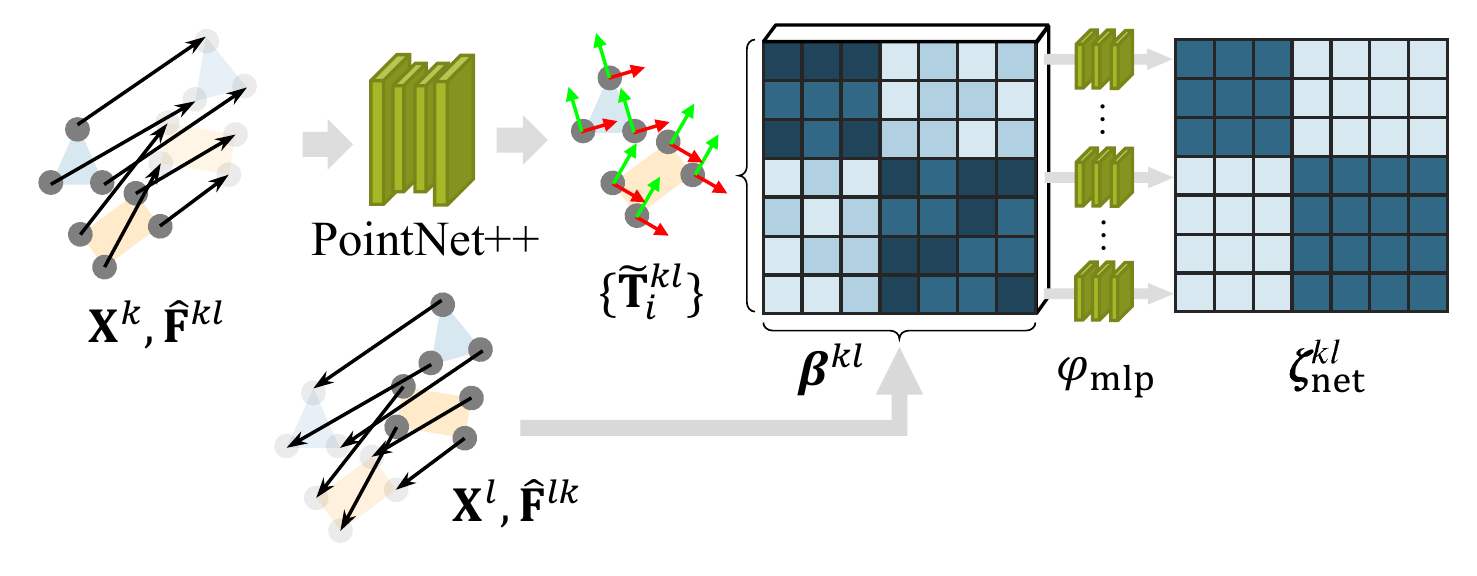}
    \caption{Illustration of our motion segmentation network $\segnet$. \vspace{-3mm}}
    \label{fig:motion-network}
\end{figure}

Our motion segmentation network $\segnet(\cdot):\R^{12\times N}\mapsto\R^{N\times N}$ illustrated in~\cref{fig:motion-network} takes the point cloud pair $\X^k$, $\X^l$ as well as flow $\hat{\flow}^{kl},\hat{\flow}^{lk}$ estimated from the last step  as input and outputs the matrix $\Z^{kl}$. 
It begins with a PointNet++~\cite{qi2017pointnet++} predicting a transformation $\Tp^{kl}_i$ for each point in $\x_i^k\in\X^k$\footnote{In practice, instead of predicting $\Tp_i^{kl}$ directly, we estimate a residual motion \wrt the already obtained flow vectors similar to the method in \cite{yi2018deep}. This procedure is detailed in our supplementary material.}.
The predictions map the part in $\X^k$ containing $\x_i^k$ to $\X^l$.
We then compute a residual matrix $\bm{\beta}^{kl} \in \R^{3 \times N\times N}$ based on $\Tp_i^{kl}$, whose element is:
\begin{equation}
    \bm{\beta}_{ij}^{kl} = (\Tp_i^{kl})^{-1} \circ \x_j^l - (\x_j^l + \hat{\fpt}_j^{lk}),
\end{equation}
where $\circ$ denotes the action of $\Tp$. 
One can easily verify that the smaller the norm of the $(i,j)$-th entry of $\bm{\beta}^{kl}$ is, the more likely that $\x_i^k$ and $\x_j^l$ are in the same rigid part. 
Therefore, it contains valuable information for deducing the motion segmentation $\Z^{kl}$. 
We apply $N$ denoising mini-PointNet~\cite{qi2017pointnet} $\denoisenet(\cdot)$ to each horizontal $3\times N$ slice of $\bm{\beta}^{kl}$, concatenated with $\X^l$ to get a likelihood score for each pair of points $(\x_i^k\in\X^k,\x_j^l\in\X^l)$. The network output $\Znet^{kl}$ is subsequently computed by applying a sigmoid on the output:
\begin{align}
\label{equ:denoise}
    (\Znet^{kl})_{i,:} = \mathrm{sigmoid}\left(\denoisenet([\bm{\beta}^{kl}_{i,:}, \X^l])\right).
\end{align}

\parahead{Motion segmentation consistency}
Given all pairwise motion information {$\Z^{kl}$}, we adopt the method of Arrigoni and Pajdla~\cite{arrigoni2019motion} to compute an absolute motion segmentation $\Gs \in \R^{KN \times S}$ as a stack of matrices in $\mathcal{G}$.
Once again, this is an instance of a synchronization problem, with the stacked relative and absolute motion segmentation matrices being: 
\begin{equation}
\Zm = \begin{bmatrix}
\zero & \Z^{12} & \dots & \Z^{1K} \\
\Z^{21} & \zero & \dots & \Z^{2K} \\
\vdots & \vdots & \ddots & \vdots \\
\Z^{K1} & \Z^{K2} & \dots & \zero \\
\end{bmatrix}, \quad
\Gs = \begin{bmatrix}
\G^1 \\
\G^2 \\
\vdots \\
\G^K \\
\end{bmatrix},
\end{equation}
A spectral approach similar to the one in~\cref{subsec:flow} optimizes for $\Gs$ so that $\Zm = \Gs \Gs^\top$ is best satisfied.
Then, $\Gs$ is just the $S$ leading eigenvectors of $\Zm$, scaled by the square root of its $S$ largest eigenvalues. 
Here, the point-part association matrices $\G^k$ are relaxed to \emph{fuzzy} segmentations by allowing its entries to take real values.
As a subsequent step similar to~\cref{subsec:flow}, we replace the projection step with a row-wise softmax on $\Gs$ to maintain differentiability.

Note that the output $\Znet^{kl}$ of $\varphi_{mot}$ is \emph{unnormalized}, meaning that any submatrix in $\Zm$ can be written as $\Z^{kl}=\sigma^{kl}\Znet^{kl}$, where $\sigma$ acts as a normalizer. This is akin to encoding a \emph{confidence} in the norm of the matrix $\Znet^{kl}$ and requires us to solve a weighted synchronization. However, as we prove in the following theorem, such a solution would involve an anisotropic scaling in the eigenvectors as a function of the number of points belonging to each part. As this piece of information is not available in runtime, we take an alternative approach and approximate the scaling factor as $q^{kl}=\mathrm{mean}(\Znet^{kl})$ and \emph{pre-factor} it out of $\Znet^{kl}$, by letting $\Z^{kl} = \Znet^{kl} / q^{kl}$. In this way, we ensure that the eigenvectors yield the synchronized motion segmentation.

\begin{thm}
Under mild assumptions, the solution to the segmentation synchronization problem using a non-uniformly weighted matrix will result in a proportionally scaled version of the solution obtained by the eigenvectors of the unweighted matrix $\Zm$.
\end{thm}
\begin{proof}
Please refer to the supplementary material.
\end{proof}

As we show in our supplement, entry $k$ in the decomposed eigenvalues is related to the number of points belonging to motion $k$. 
To compute the number of rigid bodies $S$, \ie, determine how many eigenvectors to use in $\Gs$, the spectrum of $\Zm$ is analyzed during test time: We estimate $S$ as the number of eigenvalues that are larger than $\alpha$-percent of the sum of the first $10$ eigenvalues.
For training, we just fix $S=6$ as an over-parametrization. 

\parahead{Pose Computation and Iterative Refinement}
We finally estimate the motion for each part using a weighted Kabsch algorithm~\cite{kabsch1976solution,gojcic2020learning} followed by a joint pose estimation. During test time we also iterate our pipeline several times to gradually refine the correspondence and segmentation estimation by transforming input point clouds according to the estimated $\mathcal{T}$ and adding back the residual flow onto the flow predicted at the previous iteration. The details are provided in our supplementary.

\begin{figure*}[t]
    \centering
    \includegraphics[width=\linewidth]{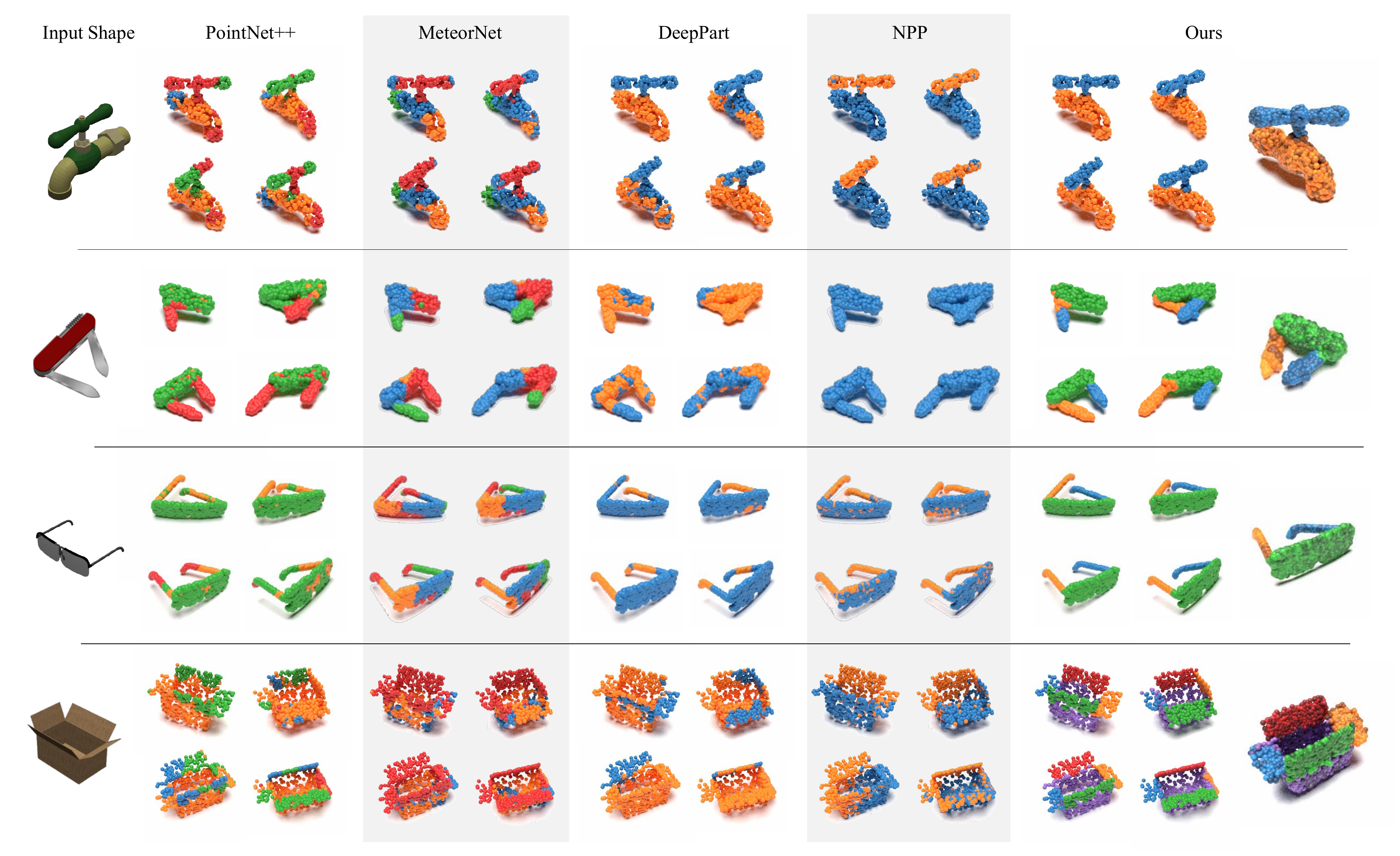}
    \caption{Qualitative results on SAPIEN~\cite{Xiang_2020_SAPIEN} dataset. On the left-most column, we show reference rendering of the objects we tackle. Each method span two columns and the point colors show the motion segmentation. For our method, we additionally show the registration result as the last column, where the darkness of the points shows the point cloud index it comes from.\vspace{-3mm}}
    \label{fig:sapien}
\end{figure*}
\subsection{Network Training}
\label{subsec:train}
\vspace{-3pt}
We propose to train each learnable component of our pipeline separately in a pairwise manner and then fine-tune their parameters using the full pipeline.
Specifically, we first train the flow estimation network $\flownet$ supervised with ground-truth flow: $
    \mathcal{L}_\mathrm{flow}^{kl} = \left \lVert \flow^{kl} - \flow^{kl,\gt} \right \rVert_F^2$. 
Given the trained $\flownet$, the confidence estimation network $\confnet$ is trained based on its output using a binary cross-entropy (BCE) loss supervised by comparing whether the error of the predicted flow is under a certain threshold:
\begin{align}
    \mathcal{L}_\mathrm{conf}^{kl} = \sum_{i=1}^N \mathrm{BCE}(c_i^{kl,\gt}, c_i^{kl}),
\end{align}
with $c_i^{kl,\gt}=1$ if we have ${\lVert \fpt_i^{kl} - \fpt_i^{kl,\gt} \rVert}_2^2 < \epsilon_{\fpt}$ and 0 otherwise.
The motion segmentation network $\segnet$ is trained using joint supervision over the estimated transformation residual and the final motion segmentation matrix: $\mathcal{L}_\mathrm{seg}^{kl} = \mathcal{L}_\mathrm{trans}^{kl} + \mathcal{L}_\mathrm{group}^{kl}$ where each term is defined as:
\begin{align}
    \mathcal{L}_\mathrm{trans}^{kl} &= \sum_{i=1}^N \frac{ \sum_{j=1}^N \mathbb{I}(\zeta_{ij}^{kl}=1) \left\lVert \bm{\beta}^{kl}_{ij} \right\rVert_2^2 }{ \sum_{j=1}^N \mathbb{I}(\zeta_{ij}^{kl}=1) },\\
    \mathcal{L}_\mathrm{group}^{kl} &= \sum_{i=1}^N \sum_{j=1}^N \mathrm{BCE}(\zeta^{kl,\gt}_{ij},\zeta^{kl}_{ij}).
\end{align}

After we train all the networks (\ie, $\flownet$, $\confnet$ and $\segnet$), the entire pipeline is trained end-to-end with the supervision on both the pariwise flow $\sum_{k=1}^K \sum_{l=1}^K \mathcal{L}_\mathrm{flow}^{kl} $ and the IoU (Intersection-over-union) loss, defined as:
\begin{equation}
\label{equ:loss:iou}
    \mathcal{L}_\mathrm{iou} = \argmin_{\mathbf{A}} \sum_{s,s'=1}^{S\times S} \frac{\mathbf{A}(s,s') \cdot (\Gs^{\gt}_{:,s})^\top \Gs_{:,s'}}{ \lVert \Gs^{\gt}_{:,s} \rVert_2^2 + \lVert \Gs_{:,s'} \rVert_2^2 - (\Gs^{\gt}_{:,s})^\top \Gs_{:,s'}},\nonumber
\end{equation}
where $\mathbf{A}$ is an $S \times S$ binary assignment matrix which we found using the Hungarian algorithm.
The flow supervision is added to both the output of flow network, and the final pairwise \emph{rigid flow} computed as $\fpt_i^{kl}=\T^l_s (\T^k_s)^{-1} \circ \x_i - \x_i$.

%% file: sec4_experiment.tex
\section{Experiments}
\vspace{-3pt}
\input{table/sapien_flow}
\parahead{Datasets}
Our algorithm is tested on two main datasets: SAPIEN~\cite{Xiang_2020_SAPIEN} and \emph{\dataset} dataset contributed by this work: \textbf{SAPIEN} consists of realistic simulated articulated models with part mobility annotated. 
We ensure that the categories used for training and validation do not overlap with the test set, finally leading to 720 articulated objects with 20 different categories.
We then perform $K$ virtual 3D scans of the models, with each scan capturing the same object with a different camera (and hence object) pose and object articulating state.
Later, furthest point sampling is applied to down-sample the number of points to $N$.
\textbf{\dataset} (Dynamic Laboratory) contains 8 different scenes in a laboratory, each with 2-3 rigidly moving \emph{solid} objects from various categories.
Each of the scenes is captured 8 times, reconstructed using ElasticFusion~\cite{whelan2015elasticfusion} and between each capture, the object positions are randomly changed.
The dataset also contains manual annotations of the object segmentation mask and rigid absolute transformations.
For benchmarking, in each scene we choose different combinations of the 8 captures, leading to a total of $8\cdot \binom{8}{4}=560$ dataset items.
We believe the two different scenarios (articulated single object and moving rigid bodies) reflected in the test sets are sufficient to verify the robustness and the general applicability of our algorithm.

The training data for articulated objects are generated using the dataset from~\cite{yi2016scalable}, containing manually annotated semantic segmentation of 16 categories. 
Similar to~\cite{yi2018deep}, we generate $K$ random motions for each connected semantic part of the shapes.
For the training data of solid objects, we randomly sample independent motions for multiple objects taken from ShapeNet~\cite{chang2015shapenet} as if they are floating and rotating in the air. 
Please refer to supplementary material for detailed data specifications and visualizations.

\parahead{Metrics} Two main metrics are used:
(1) \textbf{EPE3D} (End-Point Error in 3D) of all $\binom{K}{2}$ pairs of point clouds. The mean and standard deviation (+/-) measures the rigid 3D flow estimation quality:
While the mean reflects an overall error in the transformation, the standard deviation shows how consistent the estimate is among all pairs - a desirable property in the multi-scan setting.
(2) Segmentation accuracy assesses the motion segmentation quality.
We use \textbf{mIoU} (mean Intersection-over-Union) and \textbf{RI} (Rand Index) to score the output based on `Multi-Scan' and `Per-Scan' segmentations.
For `\textbf{Multi-Scan}', we evaluate the points from all $K$ clouds altogether, revealing the consistency of the labeling across multiple scans.
For `\textbf{Per-Scan}', we compute the score for each of the clouds separately and evaluate the mean and standard deviation across all scans.

\parahead{Training} $\flownet$, $\segnet$ and $\confnet$ are trained using Adam optimizer with initial learning rate of $10^{-3}$ and a 0.5/0.7/0.7 decay every 400K iterations for the three networks. The batch sizes are set to 32/8/32, respectively.
The entire pipeline is trained end-to-end using $K=4$ point clouds, with a learning rate of $10^{-6}$.
The gradient computation for eigen-decomposition will sometimes lead to numerical instabilities~\cite{dang2018eigendecomposition}, so we roll back that iteration when the gradient norm is large.
Our algorithm is implemented using PyTorch~\cite{pytorch} with $N=512$, $\tau=0.01$, $\epsilon_{\bm{f}}=0.1$. 
We set $\alpha=0.05$ for articulated objects and $\alpha=0.15$ for solid objects.

\input{table/sapien_segm}
\subsection{Results on Articulated Objects}
\label{subsec:exp:sapien}
\vspace{-3pt}
\parahead{Baselines}
Given our new multi-scan multi-body setting, we made adaptations to previous methods and compared to the following 4 baselines:
(1) \textbf{PointNet++}~\cite{qi2017pointnet++}: We use the segmentation backbone to directly predict $\G^k$ matrices. We aggregate the bottleneck features by taking the max before feeding it to the individual $K$ feature propagation modules.
(2) \textbf{MeteorNet}~\cite{liu2019meteornet}: We use the \emph{MeteorNet-seg} model proposed to directly predict the segmentations. Both PointNet++ and MeteorNet are supervised with the IoU loss (\cref{equ:loss:iou}) which counts in the ambiguity of rigid body labeling.
(3) \textbf{DeepPart}~\cite{yi2018deep}: As this method only allows pairwise input, we associate multiple point clouds using sequential label propagation.
(4) \textbf{NPP} (Non-Parametric Part)~\cite{hayden2020nonparametric}: This algorithm does not need training and a grid search is conveyed for its many tunable parameters.

\begin{figure}[t]
    \centering
    \includegraphics[width=\linewidth]{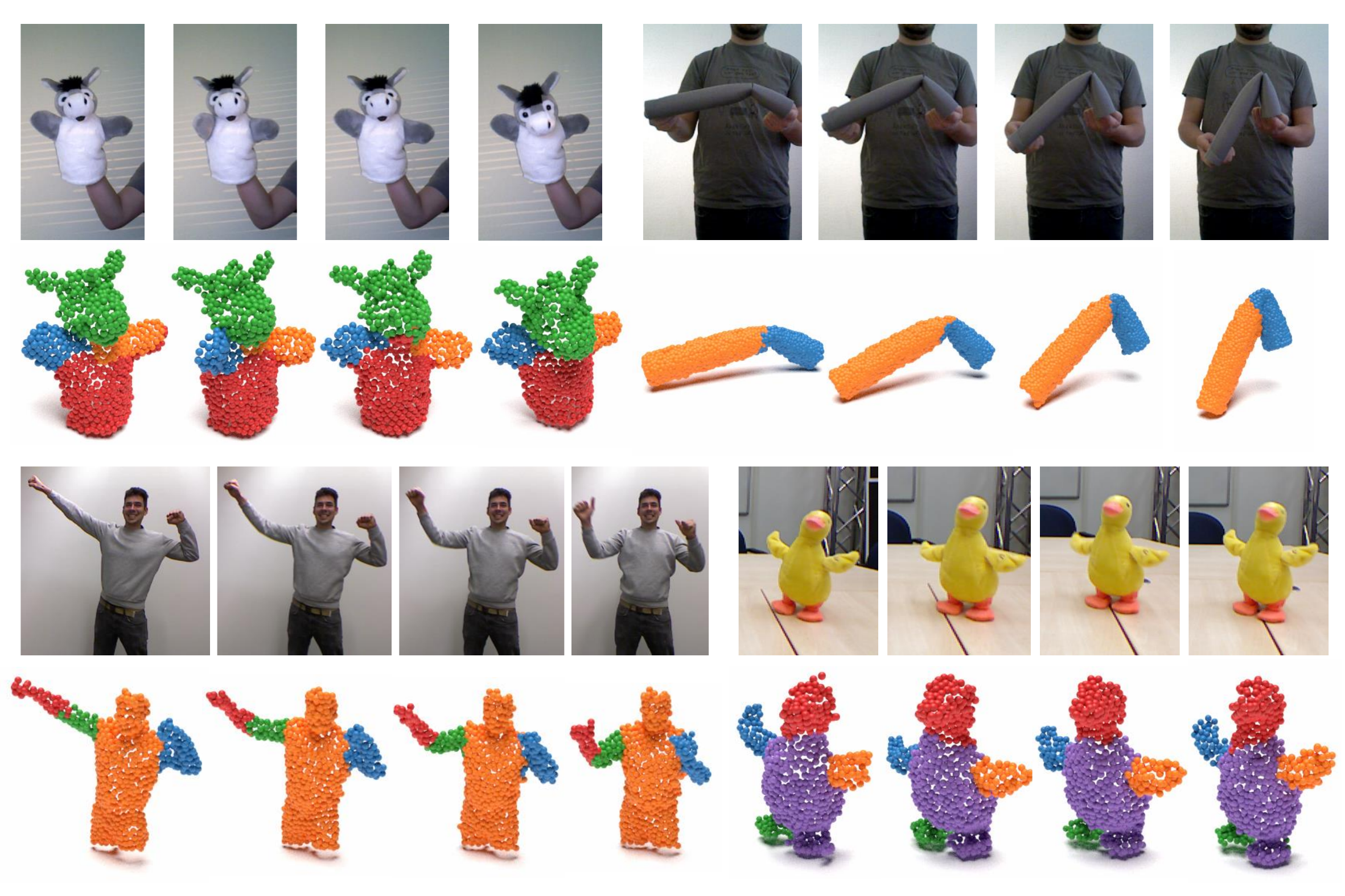}
    \caption{Qualitative segmentation results on two articulated sequences from \cite{Tzionas:ECCVw:2016} (`Donkey' and `Pipe 3/4', first two rows) and \cite{slavcheva2017cvpr} (`Alex' and `Duck', last two rows).}
    \label{fig:avped}
    \vspace{-\baselineskip}
\end{figure}

\parahead{Flow Accuracy}
\cref{tbl:sapien-flow} shows that despite being based on~\cite{yi2018deep}, our method gives the lowest flow error and variance across different view pairs. This is thanks to the correspondence consistency among the provided $K$ scans enforced by our synchronization module.
The NPP method suffers from a surprisingly high flow error mainly because the point-level correspondence is not explicitly modeled. 
Note that PointNet++ and MeteorNet are excluded because they only output point-wise segmentations.

\parahead{Segmentation Accuracy}
For the segmentation benchmark, we achieve a significantly better result than all the baselines as shown in \cref{tbl:sapien-segm}.
Among the baselines, MeteorNet fails because it assumes proximity of relevant data in the given point clouds, which is not robust to SAPIEN dataset because of change in both object pose and articulated parts.
Even though PointNet++ reaches a relatively high mean score, the standard deviation and Multi-Scan score show the segmentation is not consistent across different input scans.
DeepPart is specially designed for part-based motion segmentation, but only operates on two-views, which can cause drastic performance degradation if the input two-views have a large difference.
Also the RNN they proposed for part segmentation tends to generate short sequences and most of the shapes are only divided into two parts.
Despite the large error in flow estimation, NPP behaves reasonably in terms of segmentation.
Qualitative comparisons are visualized in \cref{fig:sapien}.

One important aspect of our network is that it can generalize to different objects and motions \emph{without re-training}. To qualitatively showcase this, we use two additional dynamic RGB-D sequences from \cite{Tzionas:ECCVw:2016} and \cite{slavcheva2017cvpr}.
For each sequence, we use four views and back-project the depth map into point clouds for inference. 
As shown in \cref{fig:avped}, our model trained on full objects of synthetic SAPIEN dataset, can generalize to real dynamic depth sequences producing consistent motion-based segmentation. 
This is possible thanks to the property that our network anchors on the motion and not on the specific geometry. 

\subsection{Results on Full Objects}
\label{subsec:exp:labmot}
\vspace{-3pt}
In \dataset, each rigid body (\ie object) is now semantically meaningful, so apart from the 4 baseline methods from \cref{subsec:exp:sapien}, we additionally compare to the following two alternatives:
(5) \textbf{InstSeg} (Instance Segmentation): We take the state-of-the-art indoor semantic instance module PointGroup~\cite{jiang2020pointgroup} trained on ScanNet dataset to segment for each input cloud.
(6) \textbf{Geometric}: We use the Ward-linkage~\cite{ward_hac} to {agglomeratively cluster} the points in each scan. 
In order to obtain consistent segmentation across multiple inputs, we associate the segmentations between two different scans using a Hungarian search over the object assignment matrix, whose element is the \emph{root mean squared error} measuring the fitting quality between any combinations of the object associations.

\input{table/ours_segm}

\begin{figure}[t]
    \centering
    \includegraphics[width=\linewidth]{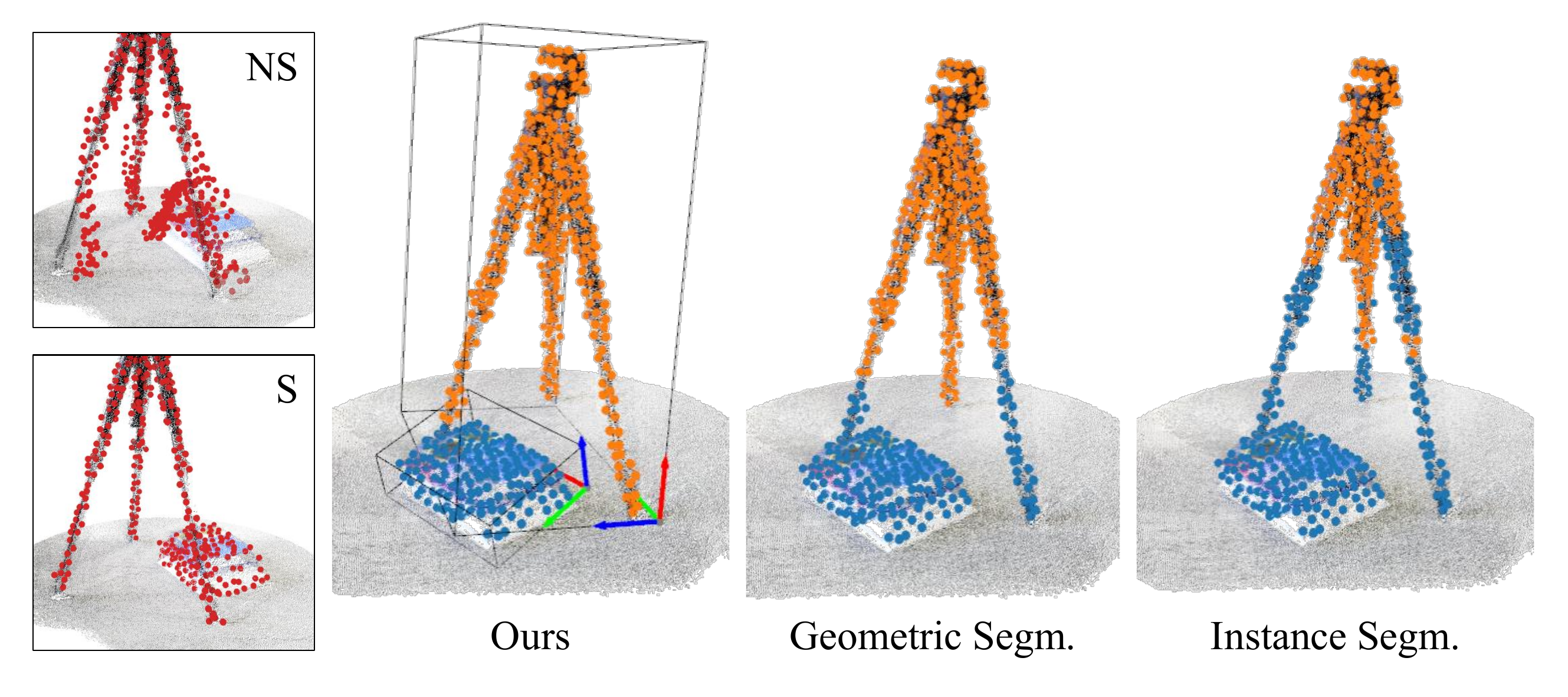}
    \caption{Example comparisons to baselines on \dataset. In the leftmost column we compare the warped point clouds without (`NS') and with (`S') synchronization. The right three sub-figures show the segmentation in different colors. For clarity we exclude the computed pose for the \emph{geometric} \& \emph{InstSeg} approaches because the inaccuracy in segmentation leads to very noisy poses.\vspace{-3mm}}
    \label{fig:full-compare}
\end{figure}

Interestingly, as listed in \cref{tbl:ours-segm}, all the previous deep methods lead to unsatisfactory results on this dataset.
PointNet++ and MeteorNet are found to be inaccurate because by design they associate labels in the level of semantics (not motion) and no explicit consistencies across scans are considered.
Even though the InstSeg method is trained on large-scale scene dataset, it is impossible for it to cover all real-world categories so wrong detections are observed in some scenes. 
The geometric approach is less robust in cluttered scenes where no obvious geometric cues can be used.
Our method is motion-induced and is hence robust to geometric variations and out-of-distribution semantics, outperforming all baselines.
A typical failure scenario for these approaches is visualized in \cref{fig:full-compare}.
We show additional qualitative results in \cref{fig:full}, demonstrating our ability to accurately segment, associate, and compute correct object transformations even if there are large pose changes.

\begin{figure}[t]
    \centering
    \includegraphics[width=\linewidth]{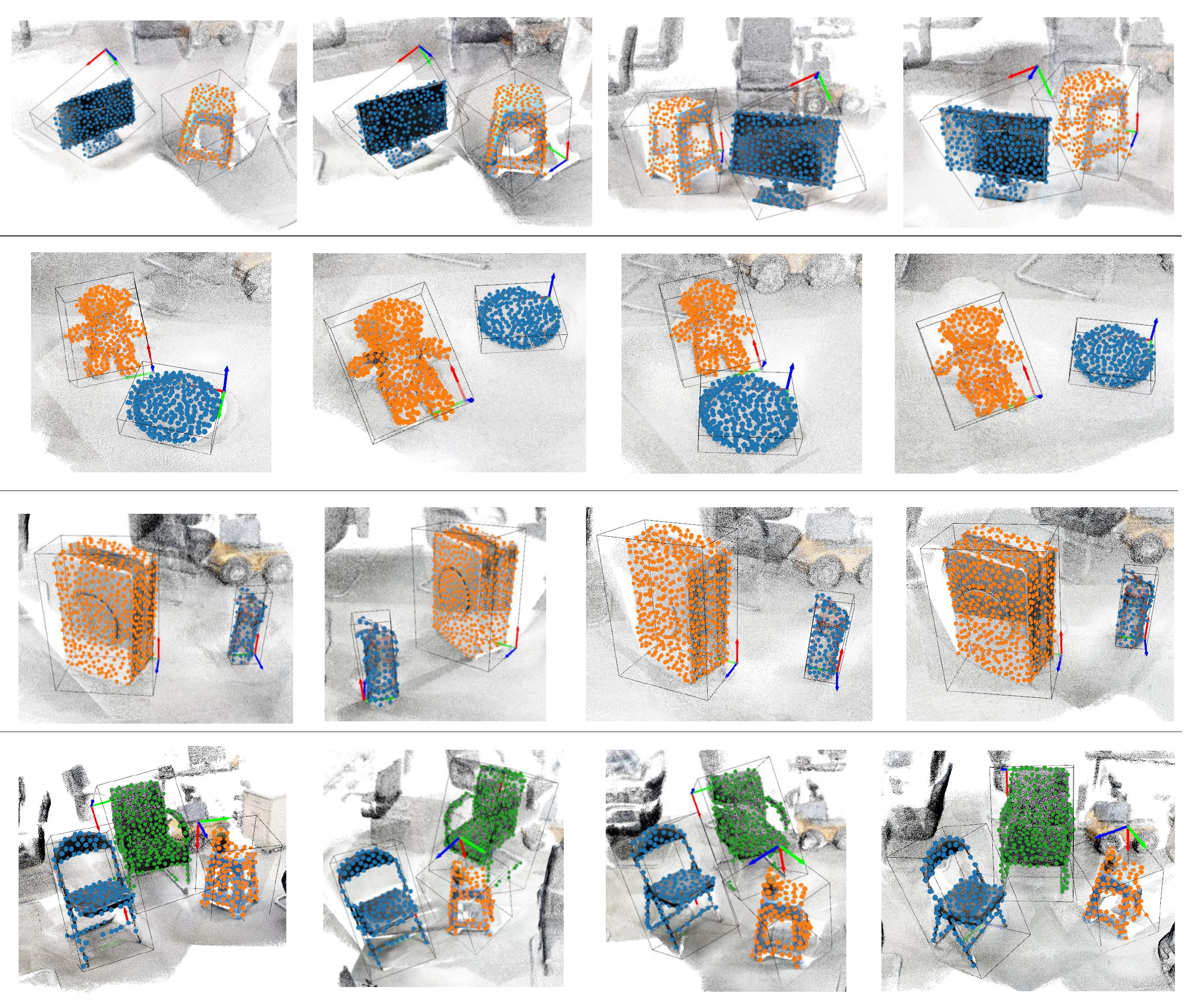}
    \caption{Results on \dataset dataset. Note that we detect and remove the ground for all baselines except InstSeg so only the points on the moving furniture are input. Point colors indicate segmentation and the bounding boxes show relative transformations.}
    \label{fig:full}
\end{figure}

\cref{tbl:ours-flow} shows the rigid flow estimation result against the baselines.
Apart from the influence of wrong per-scan segmentation and cross-scan associations, the iterative closest point (ICP)~\cite{besl1992method} method used to register object scans can also suffer from poor initializations.
Our approach not only reaches the lowest mean error, but also respects the motion consistency across multiple scans.\vspace{-1mm}
\input{table/ours_flow}

\subsection{Ablation Study and System Analysis}\label{subsec:exp:ablation}
\vspace{-3pt}
\parahead{Effect of synchronization}
For permutation synchronization (\cref{subsec:flow}), we can directly feed the network-predicted flow vector $\flow^{kl}$ to subsequent steps instead of using synchronized $\hat{\flow}^{kl}$ (Ours: NS, NW), or use an unweighted version of the synchronization by setting all $w^{kl}=1$ (Ours: S, NW). 
However, as shown quantitatively in \cref{tbl:sapien-flow}, both variants result in higher flow error due to the failure to find consistent correspondences.
Similar results can be observed on \dataset dataset as demonstrated in the two sub-figures of \cref{fig:full-compare}, where direct flow prediction failed because the geometric variation is too large between two scans.

\parahead{Effect of $K$}
Our method can be naturally applied to an arbitrary number of views $K$ even if we train using 4 views, because by design the learnable parameters are unaware of the input counts.
As shown in \cref{fig:vars}, the segmentation accuracy improves given more views. This is because the introduction of additional scans helps build the connection between existing scans and benefits the `co-segmentation' process.

\begin{figure}[t]
    \centering
    \includegraphics[width=\linewidth]{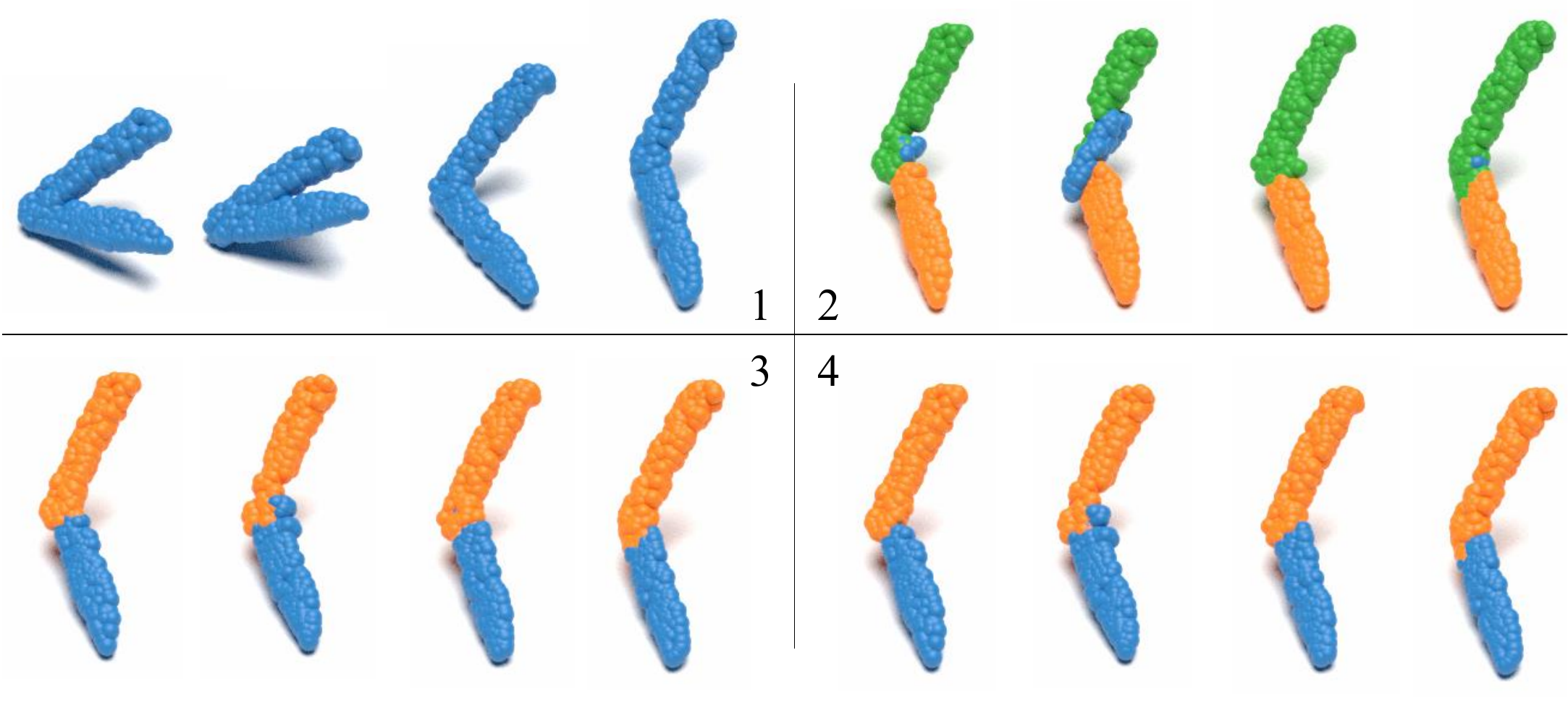}
    \caption{Iterative refinement on SAPIEN dataset. We show transformed and segmented point clouds according to recovered motions over iterations (shown in the middle).\vspace{-1mm}}
    \label{fig:iter}
\end{figure}

\parahead{Number of iterations}
As pointed out in \cref{subsec:group}, our pipeline can be run multiple iterations to refine the results and an example is given in \cref{fig:iter}.
Shown in \cref{fig:vars}, our method works better with more iterations because we estimate more accurate flows.
Moreover, more iterations are demonstrated to be unnecessary because previous iterations already lead to converged results.
\begin{figure}[!t]
    \centering
    \includegraphics[width=\linewidth]{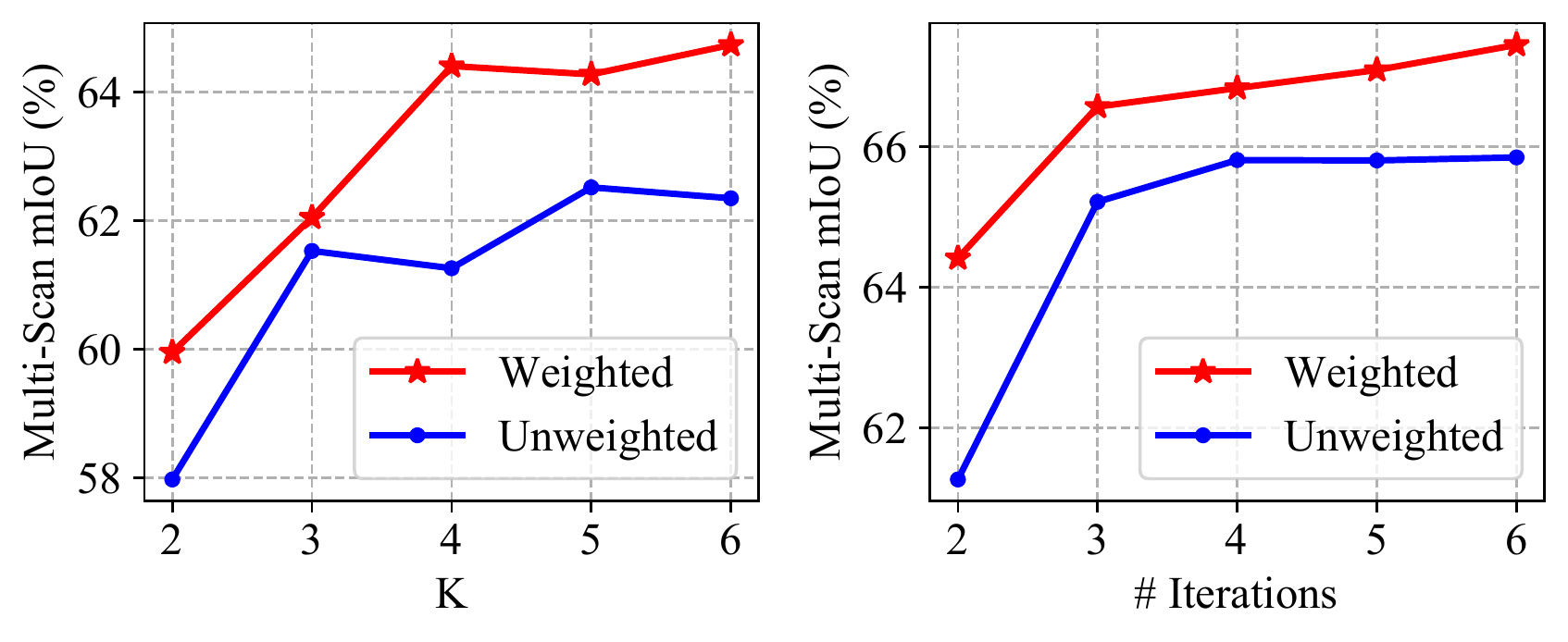}
    \caption{Influence of number of iterations (\textbf{left}) and number of views (\textbf{right}) on the final segmentation accuracy.\vspace{-4mm}}
    \label{fig:vars}
\end{figure}

\parahead{Timing}
Our experiments are conducted using an Nvidia GeForce GTX 1080 card. 
For the input of 4 scans, the running time of our full model is $\sim$870ms \emph{per iteration}. 
The entirety of a 4-iteration scheme hence takes $\sim$3.5s, while \cite{yi2018deep} and \cite{hayden2020nonparametric} take 11.5s and 60s resp. in comparison. 

%% file: table/sapien_flow.tex
\begin{table}
{
\small
\centering
\caption{Rigid flow estimation on SAPIEN. Table reports mean and std. dev. (+/-) of the EPE3D over all pairwise flows, with (S) or w/o (NS) \textbf{S}ynchronization and with (W) or w/o (NW) \textbf{W}eighting.}
\label{tbl:sapien-flow}
\begin{tabularx}{\linewidth}{c|ccCCC} 
\toprule
\multirow{2}{*}{} & \multirow{2}{*}{\footnotesize{\makecell{Deep\\Part~\cite{yi2018deep}}}} & \multirow{2}{*}{\footnotesize{NPP~\cite{hayden2020nonparametric}}} & \multicolumn{3}{c}{Ours}         \\ 
\cline{4-6}
                  &                           &                      & \footnotesize{NS, NW} & \footnotesize{S, NW} & \footnotesize{S, W}            \\ 
\midrule
Mean              & 5.95                      & 21.22                & 6.20   & 6.08  & \textbf{5.03}   \\
+/-               & 3.57                      & 6.29                 & 4.06   & 3.47  & \textbf{2.00}   \\
\bottomrule
\end{tabularx}
}\vspace{-1.6em}
\end{table}

%% file: table/sapien_segm.tex
\begin{table}
{
\small
\centering
\caption{Segmentation Accuracy on SAPIEN Dataset.}
\label{tbl:sapien-segm}
\begin{tabularx}{\linewidth}{l|CCcc} 
\toprule
\multirow{2}{*}{} & \multicolumn{2}{c}{Multi-Scan}   & \multicolumn{2}{c}{Per-Scan}                                      \\ 
\cline{2-5}
                  & mIoU           & RI              & \multicolumn{1}{c}{mIoU}       & \multicolumn{1}{c}{RI}           \\ 
\midrule
PointNet++~\cite{qi2017pointnet++}
& 47.5           & 0.62            & 51.2$\pm$12.1                  & 0.65$\pm$0.09                    \\
MeteorNet~\cite{liu2019meteornet}
& 43.7           & 0.59            & 45.7$\pm$5.4                   & 0.60$\pm$0.03                    \\
DeepPart~\cite{yi2018deep}
& 49.2           & 0.64            & 53.0$\pm$8.9                   & 0.67$\pm$0.06                    \\
NPP~\cite{hayden2020nonparametric}
& 48.2           & 0.63            & 51.5$\pm$6.6                   & 0.66$\pm$0.05                    \\ 
\midrule
\textbf{Ours} (4 iters)     & \textbf{66.7}  & \textbf{0.76 }  & \textbf{67.3}$\pm$\textbf{4.3} & \textbf{0.77}$\pm$\textbf{0.03}  \\
\bottomrule
\end{tabularx}
}\vspace{-2em}
\end{table}

%% file: table/ours_segm.tex
\begin{table}
{
\small
\centering
\caption{Segmentation Accuracy on \dataset dataset.}
\label{tbl:ours-segm}
\begin{tabularx}{\linewidth}{l|CCcc} 
\toprule
\multirow{2}{*}{} & \multicolumn{2}{c}{Multi-Scan}  & \multicolumn{2}{c}{Per-Scan}                                      \\ 
\cline{2-5}
                  & mIoU           & RI             & mIoU                           & RI                               \\ 
\midrule
PointNet++~\cite{qi2017pointnet++}
& 37.2           & 0.53           & 39.4$\pm$7.1                   & 0.54$\pm$0.03                    \\
MeteorNet~\cite{liu2019meteornet}
& 58.5           & 0.69           & 71.8$\pm$9.7                   & 0.76$\pm$0.06                    \\
DeepPart~\cite{yi2018deep}
& 60.7           & 0.70           & 66.3$\pm$17.2                  & 0.75$\pm$0.13                    \\
NPP~\cite{hayden2020nonparametric}
& 65.7           & 0.74           & 71.6$\pm$7.7                   & 0.78$\pm$0.05                    \\
Geometric         & 83.1           & 0.87           & 88.6$\pm$5.8                   & 0.91$\pm$0.04                    \\
InstSeg~\cite{jiang2020pointgroup}
& 56.5           & 0.66           & 72.4$\pm$12.5                  & 0.78$\pm$0.09                    \\ 
\midrule
\textbf{Ours}
& \textbf{90.7 } & \textbf{0.95 } & \textbf{94.0}$\pm$\textbf{3.1} & \textbf{0.96}$\pm$\textbf{0.02}  \\
\bottomrule
\end{tabularx}
}\vspace{-3mm}
\end{table}

%% file: table/ours_flow.tex
\begin{table}[!t]
{
\small
\centering
\setlength{\tabcolsep}{2.0pt}
\caption{Rigid flow estimation result on \dataset dataset. The numbers represent mean and standard deviation (+/-) of the EPE3D from all pairwise flows. Note the value here does not reflect real-world metrics because the scales are uniformly normalized.}
\label{tbl:ours-flow}
\begin{tabularx}{\linewidth}{c|cCCcC} 
\toprule
     & \footnotesize{DeepPart~\cite{yi2018deep}} & \footnotesize{NPP~\cite{hayden2020nonparametric}} & \footnotesize{Geometric} & \footnotesize{InstSeg~\cite{jiang2020pointgroup}} & Ours  \\ 
\midrule
Mean & 16.89    & 51.14 & 21.61     & 46.40                          & \textbf{11.01}            \\
+/-  & 11.39    & 18.38 & 9.76      & 20.73                          & \textbf{6.65}\\
\bottomrule
\end{tabularx}
}\vspace{-3mm}
\end{table}

%% file: sec5_conclusion.tex
\vspace{-0.5em}
\section{Conclusion}
\vspace{-3pt}
We presented \emph{MultiBodySync}, a pipeline for simultaneously segmenting and registering multiple dynamic scans with multiple rigid bodies.
We, for the first time, incorporated weighted permutation synchronization and motion segmentation synchronization into a fully-differentiable pipeline for generating consistent results across all input point clouds.
However, currently MultiBodySync is not scalable to a large number (like hundreds) of scans or rigid bodies.
Future directions include improvement of the pipeline's scalability and robustness in more complicated and dynamic settings.

\vspace{0.5em}

\begin{spacing}{0.65}
{\scriptsize
\noindent\textbf{Acknowledgements.}
We ack support from the China Scholarship Council, the Natural Science Foundation of China (No. 61521002), the Joint NSFC-DFG Research Program (No. 61761136018), a grant from Tsinghua-Tencent Joint Laboratory for Internet Innovation Technology, NSF grant CHS-1528025, a Vannevar Bush Faculty fellowship, a TUM/IAS Hans Fischer senior fellowship, and gifts from the Adobe, Amazon AWS, and Snap corporations.
Arrigoni was supported by the EU Horizon 2020 Research and Innovation Programme under project SPRING (No. 871245).}
\end{spacing}

%% file: extra/sec_content.tex
In this supplementary material, we first give the proofs of the theorems in~\cref{supp:sec:proof}, then provide more details of our implementation and our dataset in~\cref{supp:sec:detail}.
Additional ablations and results are shown in~\cref{supp:sec:results}.

%% file: extra/sec_proof.tex
\section{Proofs of Theorems}
\label{supp:sec:proof}

\subsection{Theorem 1}
\begin{proof}
The energy function in \refpaper{eq:sync} can be written as:
\begin{equation}
\begin{aligned}
E(\ps) =& \sum_{k=1}^K \sum_{l=1}^K w^{kl} \lVert \Pm^k - \Pm^{kl} \Pm^l \rVert_F^2 \\
=& \sum_{k=1}^K \sum_{l=1}^K \sum_{i=1}^N w^{kl} \lVert \Pm^k_{:i} - \Pm^{kl} \Pm^l_{:i} \rVert^2 \\
=& \sum_{i=1}^N \sum_{k=1}^K \sum_{l=1}^K w^{kl} \lVert \Pm^k_{:i} \rVert^2 + w^{lk} \lVert \Pm^l_{:i} \rVert^2 \\
&- w^{kl}(\Pm^k_{:i})^\top(\Pm^{kl} \Pm^l_{:i}) - w^{lk}(\Pm^l_{:i})^\top(\Pm^{lk} \Pm^k_{:i}) \\
=& \sum_{i=1}^N 2\sum_{k=1}^K (\Pm^k_{:i})^\top \left( \sum_{l=1}^K w^{kl} (\Pm^k_{:i} - \Pm^{kl} \Pm^l_{:i} ) \right) \\
=& \sum_{i=1}^N 2\sum_{k=1}^K (\Pm^k_{:i})^\top \left( \left(w^{k} \Id_N \right) \Pm^k_{:i} - \sum_{l\neq k} w^{kl} \Pm^{kl} \Pm^l_{:i} \right) \\
=& 2\sum_{i=1}^N \ps_{:i}^\top \Lap \ps_{:i} = 2 \mathrm{tr}( \ps^\top \Lap \ps ).\nonumber
\end{aligned}
\end{equation}

The spectral solution additionally requires each column of $\ps$ to be of unit norm and orthogonal to others relaxing $\{\Pm^{kl}\in\Man\}_{k,l}$:
\begin{equation}
    \min_{\ps} \mathrm{tr}(\ps^\top \Lap \ps) \quad\mathrm{ s.t. }\quad \ps^{\top} \ps = \Id_N.
\end{equation}

This QCQP~(Quadratically Constrained Quadratic Program) is known to have the closed form solution revealed by generalized Rayleigh problem~\cite{horn2012matrix} (or similarly, the Courant-Fischer-Weyl min-max principle).
The solution is given by the $N$ eigenvectors of $\Lap$ corresponding to the smallest $N$ eigenvalues.
\end{proof}

\subsection{Theorem 2}
We first recall the spectral solution of the synchronization problem and then extend the result to the weighted variant we propose. For completeness, here we include $\Zm=\Gs \Gs^\top$, the \textbf{unweighted} motion segmentation matrix:
\begin{equation}
\Zm = \begin{bmatrix}
\zero & \Z^{12} & \dots & \Z^{1K} \\
\Z^{21} & \zero & \dots & \Z^{2K} \\
\vdots & \vdots & \ddots & \vdots \\
\Z^{K1} & \Z^{K2} & \dots & \zero \\
\end{bmatrix}.
\end{equation}
\begin{lemma}[Spectral theorem of synchronization]
In the noiseless regime and under spectral relaxation, the synchronization problem can be cast as
\begin{equation}
    \max_\U \mathrm{tr}(\U^\top \Zm \U) \quad\mathrm{ s.t. }\quad \U^\top \U = \Id_S,
\end{equation}
where $\U \in \R^{KN \times S}$ denotes the sought solution, \ie absolute permutations. Then each column in $\U$ will be one of the $S$ leading eigenvectors of matrix $\Zm$~\cite{arrigoni2019motion}:
\begin{equation}
    \U \cdot \mathrm{diag} (\sqrt{\lambda_1}, \dots, \sqrt{\lambda_S}) \approx \Gs = \begin{bmatrix}
\G^1 \\
\G^2 \\
\vdots \\
\G^K \\
\end{bmatrix},
\end{equation}
where $\lambda_1, \dots, \lambda_S$ are the leading eigenvalues of $\Zm$.
\end{lemma}
We now recall the weighted synchronization problem. 
Here we assume the $\Z^{kl}$ matrices are binary and satisfy the properties listed in \cite{arrigoni2019motion}. The weighted synchronization matrix $\tilde{\Zm}$ is composed of a set of anisotropically-scaled $\Z^{kl}$ matrices:
\begin{equation}\label{eq:ZW}
\ZW = \begin{bmatrix}
\zero & \frac{1}{\sigma^{12}}\Z^{12} & \dots & \frac{1}{\sigma^{1K}}\Z^{1K} \\
\frac{1}{\sigma^{21}}\Z^{21} & \zero & \dots & \frac{1}{\sigma^{2K}}\Z^{2K} \\
\vdots & \vdots & \ddots & \vdots \\
\frac{1}{\sigma^{K1}}\Z^{K1} & \frac{1}{\sigma^{K2}}\Z^{K2} & \dots & \zero \\
\end{bmatrix}.
\end{equation}

Remind that in the main paper we use the \emph{unweighted} synchronization (\ie without $\frac{1}{\sigma}$) by cancelling the effect of the weights via a normalization.~\cref{thm:weightedSpectral}, which we now state more formally, is then concerned about the linear scaling of the solution proportional to the weights in the motion segmentation matrix:
\setcounter{thm}{1}
\begin{thm}[Weighted synchronization for segmentation]
\label{thm:weightedSpectral}
The spectral solution to the weighted version of the synchronization problem
\begin{equation} \label{eq:uzu}
    \max_{\UW} \mathrm{tr}(\UW^{\top} \tilde{\Zm} \UW) \quad\mathrm{ s.t. }\quad \UW^{\top} \UW = \Id_S
\end{equation}
is given by the columns of $\Gw$: 
\begin{equation}\label{eq:Ug}
    \UW \cdot \mathrm{diag} (\sqrt{\evalw_1}, \dots, \sqrt{\evalw_S}) \approx \Gw = \begin{bmatrix}
\G^1 \D^1 \\
\G^2 \D^2 \\
\vdots \\
\G^K \D^K \\
\end{bmatrix},
\end{equation}
Here $\evalw_1, \dots, \evalw_S$ are the leading eigenvalues of $\tilde{\Zm}$, and $(\D^1, \dots, \D^K \in \R^{S\times S})$ are diagonal matrices. 
In other words, the columns of $\Gw$ being the eigenvectors of $\ZW$ are related to the non-weighted synchronization by a piecewise linear anisotropic scaling.
\end{thm}
\begin{proof}
We begin by the observation that $\K^k=\G^{k\top}\G^k$ is a diagonal matrix where $K^k_{ss}$ counts\footnote{According to our assumption, this `count' hereafter is only valid when $\Z^{kl}$s are binary and can be viewed as \emph{soft counting} when such an assumption is relaxed.} the number of points in point cloud $k$ belonging to part $s$. Hence, each element along $\Gs^\top\Gs=\sum_{k=1}^K (\K^k) $ counts the number of points over all point clouds that belong to part $s$. Because $\Zm=\Gs\Gs^\top$, we have the following spectral decomposition $\Zm\Gs=\Gs\evals$~\cite{arrigoni2019motion}:
\begin{equation}
    \Zm\Gs = \Gs\Gs^\top\Gs = \Gs \sum\limits_{k=1}^K \G^{k\top}\G^k = \Gs \evals.
\end{equation}

To simplify notation we overload $w^{kl}$ by setting $w^{kl}=\frac{1}{\sigma^{kl}}$ for the rest of this subsection.
Let us now write $\ZW\Gw$ in a similar fashion and seek the similar emergent property of eigen-decomposition:
\begin{align}\label{eq:ZWGw}
    \ZW\Gw = \begin{bmatrix}
\sum\limits_{l=1}^K w^{1l} \Z^{1l} \G^l \D^l \\
\sum\limits_{l=1}^K w^{2l} \Z^{2l} \G^l \D^l \\
\vdots \\
\sum\limits_{l=1}^K w^{Kl} \Z^{Kl} \G^l \D^l \\
\end{bmatrix}.
\end{align}
Then, using $\Z^{kl}=\G^k\G^{l\top}$ we can express~\cref{eq:ZWGw} as:
\begin{align}\label{eq:Zg}
    \ZW\Gw &= \begin{bmatrix}
\sum\limits_{l=1}^K w^{1l} \G^1\G^{l\top} \G^l \D^l \\
\sum\limits_{l=1}^K w^{2l} \G^2\G^{l\top} \G^l \D^l \\
\vdots \\
\sum\limits_{l=1}^K w^{Kl} \G^K\G^{l\top}\G^l \D^l \\
\end{bmatrix}\\
&= \begin{bmatrix}
\G^1\sum\limits_{l=1}^K w^{1l} \G^{l\top} \G^l \D^l \\
\G^2\sum\limits_{l=1}^K w^{2l} \G^{l\top} \G^l \D^l \\
\vdots \\
\G^K\sum\limits_{l=1}^K w^{Kl} \G^{l\top}\G^l \D^l \\
\end{bmatrix} = \begin{bmatrix}
\G^1 \Hw^1 \\
\G^2 \Hw^2 \\
\vdots \\
\G^K \Hw^K
\end{bmatrix}\label{eq:GH}
\end{align}
where:
\begin{align}
    \Hw^k = \sum\limits_{k=1}^K w^{kl} \G^{l\top} \G^l \D^l.
\end{align}
$\Hw$ is a diagonal matrix because $\D^l$ is diagonal by assumption.
Note that, the first part in the summation is assumed to be a \emph{known}\footnote{We will see later in~\cref{rem:known} why this is only an assumption.} diagonal matrix (see the beginning of proof):
\begin{align}\label{eq:EH}
    \E^{kl}=w^{kl} \G^{l\top} \G^l,
\end{align} 
This form is very similar to~\cref{eq:Ug} scaled by the corresponding diagonal matrices. 
Let us know consider the $s^{\mathrm{th}}$ column of $\Gw$ responsible for part $s$. We are interested in showing that such column is an eigenvector of $\ZW$: 
\begin{align}\label{eq:Zgeigen}
    \ZW \Gw^s = \evalw_s  \Gw^s.
\end{align}
In other words, we seek the existence of $\evalw_s$ such that~\cref{eq:Zgeigen} is satisfied. Moreover, a closed form expression of $\evalw_s$ would allow for the understanding of the effect of the weights on the problem. Let us now plug~\cref{eq:Ug} and~\cref{eq:GH} into~\cref{eq:Zgeigen} to see that:
\begin{align}\label{eq:Zgeigen2}
    \begin{bmatrix} (\G^1 \Hw^1)^s \\
    (\G^2 \Hw^2)^s \\
    \vdots \\
    (\G^K \Hw^K)^s 
    \end{bmatrix} = \evalw_s \begin{bmatrix} (\G^1 \D^1)^s \\
    (\G^2 \D^2)^s \\
    \vdots \\
    (\G^K \D^K)^s 
    \end{bmatrix}.
\end{align}
As $\G^k$ is a binary matrix, it only actas as a column selector, where for a single part $s$, a column of the motion segmentation $\Gw$ should contain only ones. We can use this idea and the diagonal nature of $\ZW \Gw^s$ to cancel $\G^k$ on each side. Re-arranging the problem in terms of scalars on the diagonal yields:
\begin{align}\label{eq:sys}
\systeme*{H^1_{ss}= \sum\limits_{l=1}^K E^{1j}_{ss}D^{l}_{ss}=\evalw_s D^1_{ss}, H^2_{ss}=\sum\limits_{l=1}^K E^{2j}_{ss}D^{l}_{ss}=\evalw_s D^2_{ss}, &&\mathrel{\makebox[\widthof{=}]{\vdots}}\hfill, H^K_{ss}= \sum\limits_{l=1}^K E^{Kj}_{ss}D^{l}_{rr}=\evalw_s D^K_{ss}}    
\end{align}
where $E$ is as defined in~\cref{eq:EH}.
Note that both $D$ and $\evalw_s$ are unknowns in this seemingly non-linear problem. Yet, we can re-arrange
~\cref{eq:sys} into another eigen-problem:
\begin{align}\label{eq:sysJEigen}
    \J^s\dvec^s=\evalw_s^\prime\dvec^s,
\end{align} 
where:
\begin{align}\label{eq:sysJ}
    \J^s = \begin{bmatrix}
    E^{11}_{ss} & E^{12}_{ss} & \cdots & E^{1K}_{ss}\\
    E^{21}_{ss} & E^{22}_{ss} & \cdots & E^{2K}_{ss}\\
    \vdots & \ddots & \vdots\\
    E^{K1}_{ss} & E^{K2}_{ss} & \cdots & E^{KK}_{ss}
    \end{bmatrix} \, \dvec^s = \begin{bmatrix}
    D^1_{ss}\\
    D^2_{ss}\\
    \vdots \\
    D^K_{ss}
    \end{bmatrix}. 
\end{align}

Hence, we conclude that the eigenvectors of the weighted synchronization have the form of~\cref{eq:sysJEigen} if and only if we can solve~\cref{eq:Ug}. This is possible as soon as $\E^{kl}$ are known and $\J^s$ has real eigenvectors. Besides an \emph{existence} condition,~\cref{eq:Ug} also provides an explicit closed form relationship between the weights and the eigenvectors once $\E^{kl}$ are available.


\end{proof}
\begin{remark}\label{rem:known}
Note that the symmetric eigen-problem given in~\cref{eq:sysJ} only requires the matrix $\E^{kl}$ for all $k,l$. By definition, each element along the diagonal of $\E^{kl}=w^{kl} \G^{k\top} \G^l$ denotes the number of points in each point cloud belonging to each part weighted by $w$. Hence, it does not require the complete knowledge on the part segmentation but only the amount of points per part. While this is unknown in practice, for the sake of our theoretical analysis, we might assume the availability of this information. Hence, we could speak of solving~\cref{eq:sysJEigen} for each part $s$.
\end{remark}
\begin{remark}
It is also interesting to analyze the scenario where one assumes $\dvec^s=\one$ for each $s$. In fact, this is what would happen if one were to naively use the unweighted solution for a weighted problem, \ie use $\Gw$ itself as the estimate of motion segmentation, as our closed form expression for $\D^k$ (\cref{eq:sysJ}) cannot be evaluated in test time. Then, assuming $\D^k$ to be the identity, for each $k$ it holds:
\begin{align}
    \sum\limits_{l=1}^K E^{kl}_{ss} &= \sum\limits_{l=1}^K w^{kl} \G^{k\top}(\G^l)^s \\
    &= w^{k1}\G^{1\top}(\G^1)^s + \dots + w^{kK}\G^{K\top}(\G^K)^s \nonumber \\
    &= \w_k \cdot \begin{bmatrix} \G^{1\top}(\G^1)^s &  \cdots & \G^{K\top}(\G^K)^s \end{bmatrix} \nonumber \\
    &= \w_{k} \bm{\varphi}^s = \evalw_s^\prime \label{eq:wlambda}.
\end{align}
where $(\G^l)^s$ is the $s$-th column of $\G^l$.
The final equality follows directly from~\cref{eq:sys} when $D_{ss}=1$.
Note that we can find multiple weights $\w_{k}$ satisfying~\cref{eq:wlambda}. For instance, if $\bm{\varphi}$ and $\lambda$ were known, one solution for any $s$ would be:
\begin{equation}
    w^{kl} = \frac{\evalw_s}{K\varphi_k^s}.
\end{equation}
Because (i) we cannot assume a uniform prior on the number of points associated to each part and (ii) it would be costly to perform yet another eigendecomposition, we choose to cancel the effect of the predicted weights $w_{ij}$ as we do in the paper by a simple normalization procedure. However, such unweighted solution would only be possible because our design encoded the weights in the norm of each entry in the predicted $\Znet^{kl}$.
\end{remark}

%% file: extra/sec_detail.tex
\section{Implementation Details}
\label{supp:sec:detail}

\subsection{Network Structures}

\subsubsection{Flow Prediction Network}
We adapt our own version of flow prediction network $\flownet$ from PointPWC-Net~\cite{wu2019pointpwc} by changing layer sizes and the number of pyramids.
As illustrated in \cref{fig:pointpwc}, the network predicts 3D scene flow in a coarse-to-fine fashion.
Given input $\X^k$ as source point cloud and $\X^l$ as target point cloud, a three-level pyramid is built for them using furthest point sampling as $\{ \X^{k,(0)}=\X^k, \X^{k,(1)}, \X^{k,(2)} \}$ and $\{ \X^{l,(0)}=\X^l, \X^{l,(1)}, \X^{l,(2)} \}$, with point counts being 512, 128, 32, respectively.
Similarly, we denote the flow predicted at each level as $\{ \flow^{kl,(0)},\flow^{kl,(1)},\flow^{kl,(2)} \}$.
Per-point features for all points are then extracted with dimension 128, 192 and 384 for each hierarchy.
A 3D `Cost Volume'~\cite{kendall2017end} is then computed for the source point cloud by aggregating the features from $\X^k$ and $\X^l$ for the point pyramid, with feature dimension 64, 128 and 256.
This aggregation uses the neighborhood information relating the target point cloud and the warped source point cloud in a patch-to-patch manner.
The cost volume, containing valuable information about the correlations between the point clouds, is fed into a scene flow prediction module for final flow prediction.
The predicted flow at the coarser level can be upsampled via interpolation and help the prediction of the finer level.
Readers are referred to~\cite{wu2019pointpwc} for more details.

\begin{figure}[t]
    \centering
    \includegraphics[width=\linewidth]{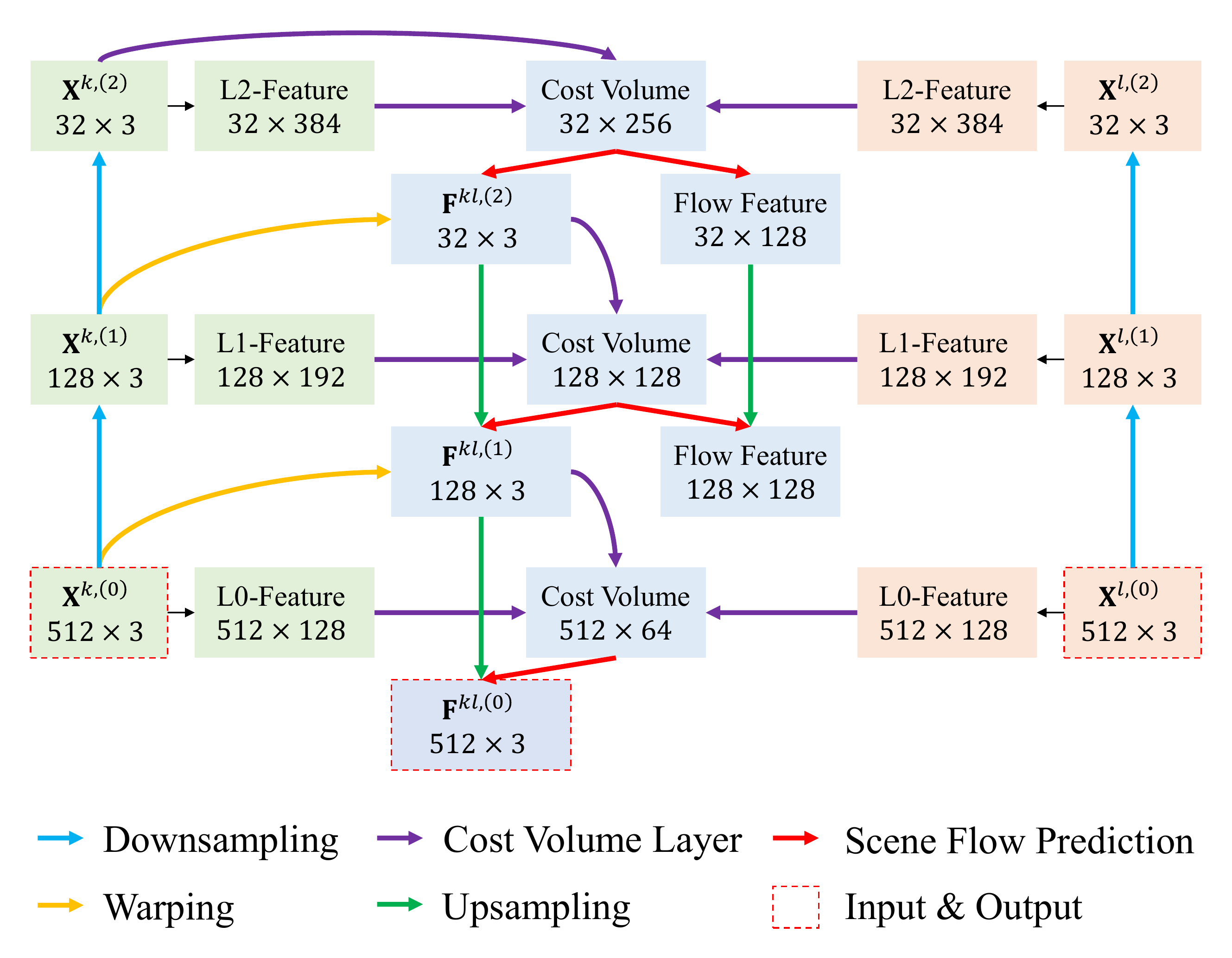}
    \caption{Our adapted version of PointPWC-Net $\flownet$. Each rectangular block denotes a tensor, whose size is written as $N\times C$ (batch dimension is ignored) below its name, with $N$ being the number of points and $C$ being the feature dimension. The network is composed of 3 hierarchical levels. At each level, features from the two input point clouds are fused via a Cost Volume Layer, which digests warped point cloud and features from the upsampled coarse flow estimated from the last level and provides a cost volume for flow prediction.}
    \label{fig:pointpwc}
\end{figure}

\subsubsection{Confidence Estimation Network}

The confidence estimation network $\confnet$ we use, adapted from OANet (Order-Aware Network)~\cite{zhang2019learning}, learns inlier probability of point correspondences.
In our case, each correspondence is represented as a $\R^7$ vector as described in the main paper.
Different from other network architectures like PointNet~\cite{qi2017pointnet}, OANet features in the novel differentiable pooling (DiffPool) and unpooling (DiffUnpool) operations as well as the order-aware filtering block, which are demonstrated to effectively gather local context and are hence useful in geometric learning settings, especially for outlier rejection~\cite{bian2017gms}.

The network starts and ends with 6 PointCN~\cite{moo2018learning} layers, which globally exchanges the point feature information by context normalization (\ie whitening along the channel dimension to build cross-point relationship).
In between the PointCNs lies the combination of DiffPool layer, order-aware filtering block and DiffUnpool layer.
The DiffPool layer learns an $N\times M$ soft assignment matrix, where each row represents the classification score of each point being assigned to one of the $M$ `local clusters'.
These local clusters represent local structures in the correspondence space and are implicitly learned.
As the $M$ clusters are in canonical order, the feature after the DiffPool layer is permutation-invariant, enabling the order-aware filtering block afterward to apply normalization along the spatial dimension (\ie, `Spatial Correlation Layer') for capturing a more complex global context. 
In our $\confnet$, we choose $M=64$.

\subsubsection{Motion Segmentation Network}
The architecture of $\segnet$ has been already introduced in the main paper.
Here we elaborate how the transformations $\Tp_i^{kl}$ are estimated by PointNet++.
The input to the network is the stacked $[ (\X^k)^\top \; (\hat{\flow}^{kl})^\top ]^\top \in \R^{6\times N}$ and the output is in $\R^{12 \times N}$, where for each point we take the first 9 dimensions as the elements in the rotation matrix and the last 3 dimensions as the translation vector.

In practice, direct transformation estimation from the PointNet++ is not accurate. 
Given that we have already obtained the flow vectors, instead of estimating $\Tp^{kl}_i$, we compute a residual motion \wrt the given flow similar to the method in \cite{yi2018deep}.
Specifically, when the actual outputs from the network are $\Rot_\mathrm{net} \in \R^{3\times 3}$ and $\tra_\mathrm{net} \in \mathbb{R}^{3}$, the transformations used in subsequent steps of the pipeline $\Tp^{kl}_i = [ \Rp_i^{kl} | \tp^{kl}_i ]$ are computed as follows:
\begin{equation}
    \Rp_i^{kl} = \Rot_\mathrm{net} + \Id_3, \quad
    \tp^{kl}_i = \tra_\mathrm{net} - \Rot_\mathrm{net} \x^k_i + \bm{f}^{kl}_i.
\end{equation}
Note that we do not ensure $\Tp_i^{kl}$ is in SE(3) with SVD-like techniques.
In fact the transformation is not directly supervised (neither in this module nor in the entire pipeline) and the nearest supervision comes from $\bm{\beta}^{kl}$ matrix through Eq (9).
This avoids the efforts to find a delicate weight for balancing the rotational and translational part of the transformation.

\subsection{Pose Computation and Iterative Refinement}
\label{supp:subsec:iter}

Given the synchronized pairwise flow $\hat{\bm{f}}^{kl}$ and motion segmentation $\G^k$, we estimate the motion separately for each rigid part using a weighted Kabsch algorithm~\cite{kabsch1976solution}.
The weight for point $\x_i^k$ and the rigid motion $s$ between $\X^k$ and $\X^l$ is taken as $c^{kl}_i G^k_{is}$.
We then use similar techniques as in~\cite{gojcic2020learning,huang2019learning} to estimate the motions separately for each part.

\begin{figure}[t]
    \centering
    \includegraphics[width=\linewidth]{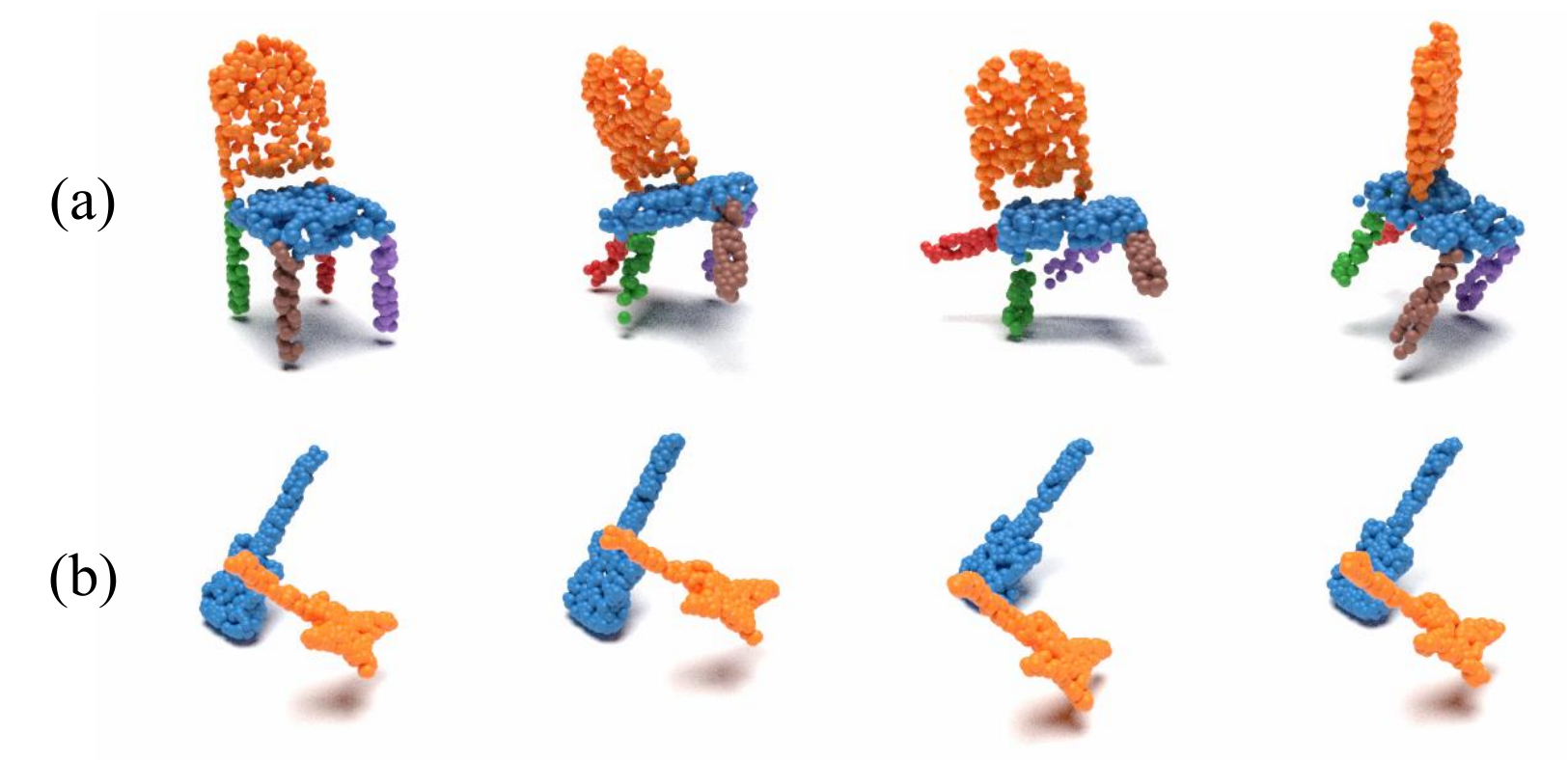}
    \caption{Examples from our training set for (a) articulated objects and (b) multiple solid objects. Different colors indicate rigidly moving parts.}
    \label{fig:training-vis}
\end{figure}

\input{table/shapenet_training}

\input{table/whole_training}

\begin{figure*}
    \centering
    \includegraphics[width=\linewidth]{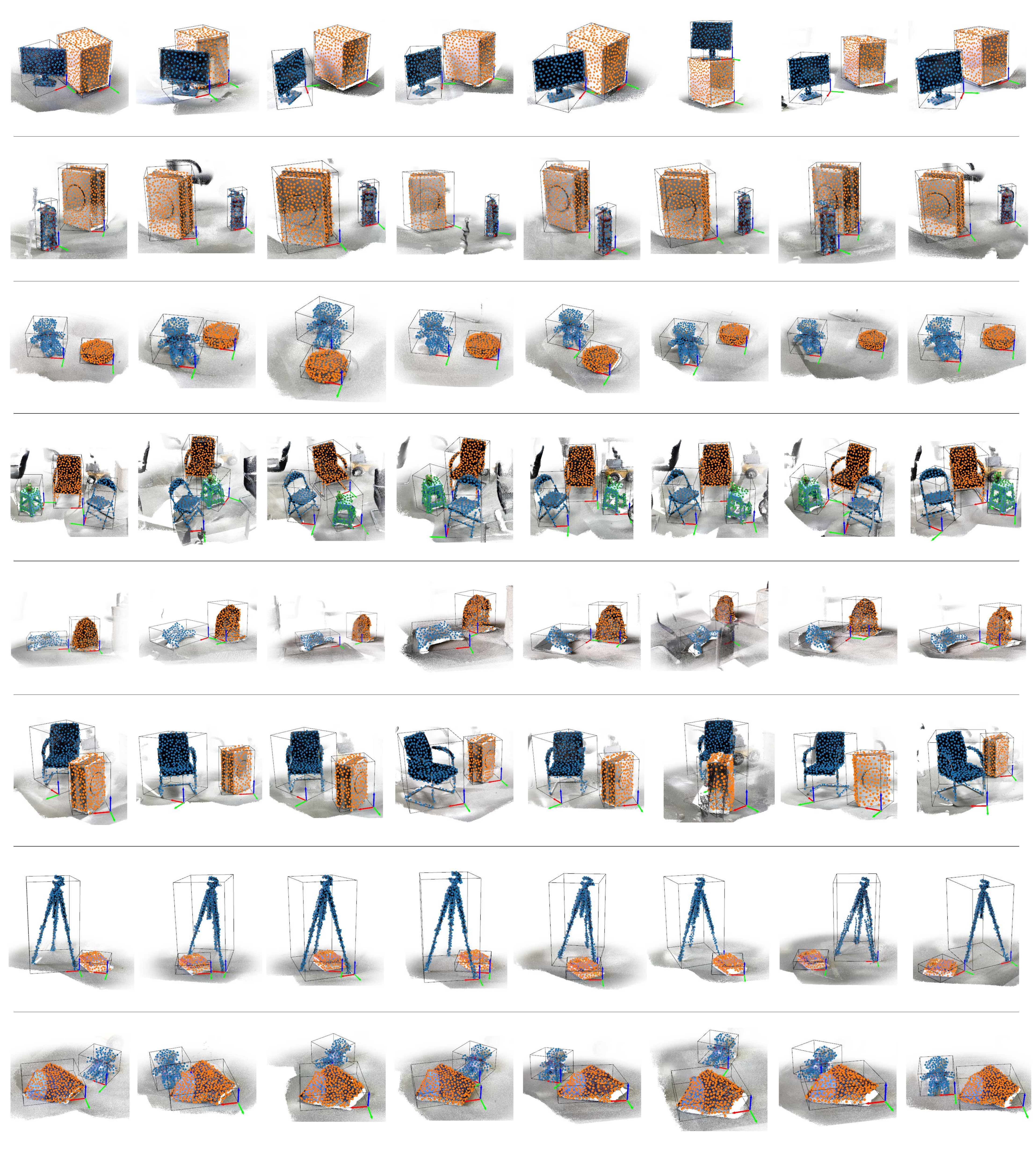}
    \caption{Visualization of the \dataset dataset. Each row shows 8 different dynamic configurations of the same set of rigid objects. Annotated bounding boxes are parallel to the ground plane and reflect the objects' absolute poses.}
    \label{fig:dynlab-vis}
\end{figure*}

The point clouds to register can have a large difference in poses making it hard for the flow network to recover. This might lead to wrong results in the subsequent steps.
Inspired by point cloud registration works~\cite{yi2018deep,gojcic2020learning}, during test time we iterate our pipeline several times to gradually refine the correspondence and segmentation estimation.
In particular, we use the transformation $\T_s^{k^*}(\T_s^k)^{-1}$ estimated at iteration $t-1$ to transform all the points in all point sets belonging to part $s$ to the \emph{canonical} pose of the $k^*$-th point cloud. Note that the choice of $k^\star$ is arbitrary, and we choose $k^\star=1$.
Then at iteration $t$, we feed the transformed point clouds to the flow network again to compute the residual flow, which is added back onto the flow predicted at iteration $t-1$ to form the input of the segmentation network.
The progress works reciprocally, as differences in poses of the point clouds are gradually minimized and the flow estimation will hence become more accurate, leading to better segmentation and transformations.
Specially, during the first iteration where pose differences are usually large, we treat the point clouds as if they are composed of only one rigid part to globally align the shapes.
This will provide a good pose initialization for subsequent iterations.

\subsection{Dataset}

\parahead{Training Data}
To demonstrate the generalizability of our method across different semantic categories, we ensure the categories used for training, validation and test have no overlap.
For articulated objects, the categories we use are shown in \cref{tbl:sapien-training}.
For multiple solid objects, the categories are listed in \cref{tbl:whole-training}.
Examples from our training set are visualized in \cref{fig:training-vis}.

\parahead{\dataset dataset}
A full visualization of the \dataset dataset with manual annotations is shown in~\cref{fig:dynlab-vis}.
We will make the scans publicly available.

%% file: table/shapenet_training.tex
\begin{table}

\small
\centering
\caption{Training and validation categories from \cite{yi2016scalable} used for articulated objects.}
\label{tbl:sapien-training}
\begin{tabular}{c|cccc} 
\toprule
\multirow{2}{*}{\begin{tabular}[c]{@{}c@{}}Training\\Categories\\ \end{tabular}} & {\cellcolor[rgb]{0.98,0.898,0.898}}Table     & {\cellcolor[rgb]{0.965,0.984,0.816}}Chair  & {\cellcolor[rgb]{0.796,1,0.78}}Plane   & {\cellcolor[rgb]{0.839,0.894,1}}Car        \\
                                                                                 & {\cellcolor[rgb]{0.957,0.804,0.976}}Guitar   & {\cellcolor[rgb]{0.78,1,0.973}}Bike        & {\cellcolor[rgb]{1,0.659,0.894}}Suitcase    &                                            \\ 
\midrule
\multirow{2}{*}{\begin{tabular}[c]{@{}c@{}}Validation\\Categories\end{tabular}}  & {\cellcolor[rgb]{1,0.859,0.659}}Lamp         & {\cellcolor[rgb]{0.871,0.871,0.871}}Pistol & {\cellcolor[rgb]{0.82,0.706,0.996}}Mug & {\cellcolor[rgb]{0.78,1,0.863}}Skateboard  \\
                                                                                 & {\cellcolor[rgb]{0.871,0.992,0.749}}Earphone & {\cellcolor[rgb]{0.918,0.525,0.792}}Rocket & {\cellcolor[rgb]{0.678,0.98,1}}Cap     &                                            \\
\bottomrule
\end{tabular}
\end{table}

%% file: table/whole_training.tex
\begin{table}
\small
\centering
\caption{Training and validation categories from \cite{yi2016scalable} used for multiple solid objects.}
\label{tbl:whole-training}
\begin{tabular}{c|cccc} 
\toprule
\multirow{2}{*}{\begin{tabular}[c]{@{}c@{}}Training\\Categories\\ \end{tabular}} & {\cellcolor[rgb]{0.98,0.898,0.898}}Table     & {\cellcolor[rgb]{0.816,0.859,0.984}}Knife  & {\cellcolor[rgb]{0.796,1,0.78}}Plane     & {\cellcolor[rgb]{0.839,0.894,1}}Car         \\
                                                                                 & {\cellcolor[rgb]{0.957,0.804,0.976}}Guitar   & {\cellcolor[rgb]{0.78,1,0.973}}Bike        & {\cellcolor[rgb]{1,0.659,0.894}}Suitcase & {\cellcolor[rgb]{0.961,0.953,0.718}}Laptop  \\ 
\midrule
\multirow{2}{*}{\begin{tabular}[c]{@{}c@{}}Validation\\Categories\end{tabular}}  & {\cellcolor[rgb]{1,0.859,0.659}}Lamp         & {\cellcolor[rgb]{0.871,0.871,0.871}}Pistol & {\cellcolor[rgb]{0.82,0.706,0.996}}Mug   & {\cellcolor[rgb]{0.78,1,0.863}}Skateboard   \\
                                                                                 & {\cellcolor[rgb]{0.871,0.992,0.749}}Earphone & {\cellcolor[rgb]{0.918,0.525,0.792}}Rocket & {\cellcolor[rgb]{0.678,0.98,1}}Cap       &                                             \\
\bottomrule
\end{tabular}
\end{table}

%% file: extra/sec_results.tex
\section{Additional Results}
\label{supp:sec:results}

\subsection{Extended Ablations}

In this subsection we provide more complete ablations extending \refpaper{subsec:exp:ablation}.
A full listing of the baselines we compare is as follows:
\begin{itemize}[topsep=0pt, itemsep=0pt]
    \item \textbf{Ours (1 iter)}: The pipeline is iterated only once, without the global alignment step as described in \cref{supp:subsec:iter}.
    \item \textbf{Ours (NS, NW)}: Same as the main paper, we directly feed $\flow^{kl}$ instead of $\hat{\flow}^{kl}$ to the motion network $\segnet$.
    \item \textbf{Ours (S, NW)}: Same as the main paper, we set all weights of the permutation synchronization $w^{kl}=1$.
    \item \textbf{Ours (UNZ)}: The unnormalized matrix $\ZW$~(\cref{eq:ZW}) is used as input to motion segmentation synchronization, \ie, the normalizing factors are set to $\sigma^{kl}=1$.
    \item \textbf{Ours (4 iters)}: Full pipeline of our method, with 4 steps of iterative refinement.
\end{itemize}

\begin{figure}[!t]
    \centering
    \includegraphics[width=\linewidth]{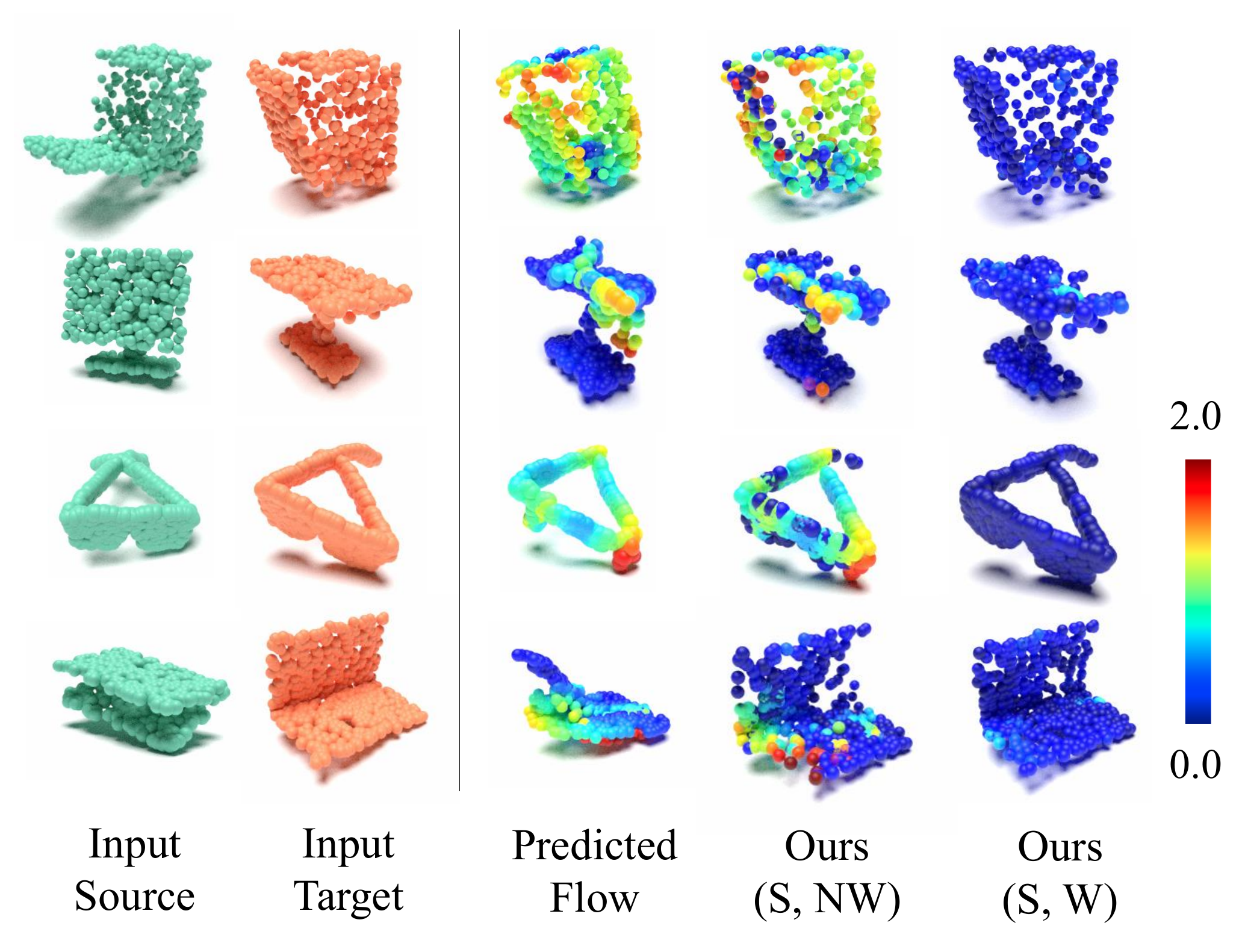}
    \caption{Visual comparisons of the pairwise flow. To visualize the flow we warp the source point cloud and compare the its similarity with the target point cloud.  The color bar on the right shows the end-point error (EPE3D). `Ours (S, W)' represents the output of our method with the \textbf{W}eighted permutation \textbf{S}ynchronization scheme.}
    \label{fig:sync}
\end{figure}

We show comparisons of the final rigid flow error using EPE3D metric on both SAPIEN and \dataset dataset in \cref{fig:ablation-flow-sapien,fig:ablation-flow-dynlab} respectively.
Results indicate that all the components introduced in our algorithm, including iterative refinement, weighted synchronization, and the pre-factoring of the motion segmentation matrix, contribute to the improvement of accuracy under different scenarios.
Note that in \dataset dataset, the performance of `Ours (UNZ)' is very similar to `Ours (4 iters)' because the motion segmentation accuracy is already high~(\refpaper{tbl:ours-segm}) due to the good quality of each individual $\Znet$ output, rendering normalization optional in practice.

We provide additional visual examples demonstrating the effectiveness of our weighted permutation synchronization in \cref{fig:sync}, where direct flow output fails due to large pose changes between the input clouds, and a naive unweighted synchronization still suffers from such failure because the influence of wrong correspondences is not eliminated.

\input{table/sapien_category}
For completeness we include per-category segmentation accuracy of articulated objects on SAPIEN~\cite{Xiang_2020_SAPIEN} dataset in \cref{tbl:sapien-category}.
The variants of our method perform consistently better than other methods for nearly all categories, showing the robustness of our model for accurate multi-scan motion-based segmentation.

\begin{figure}[t]
    \centering
    \includegraphics[width=0.8\linewidth]{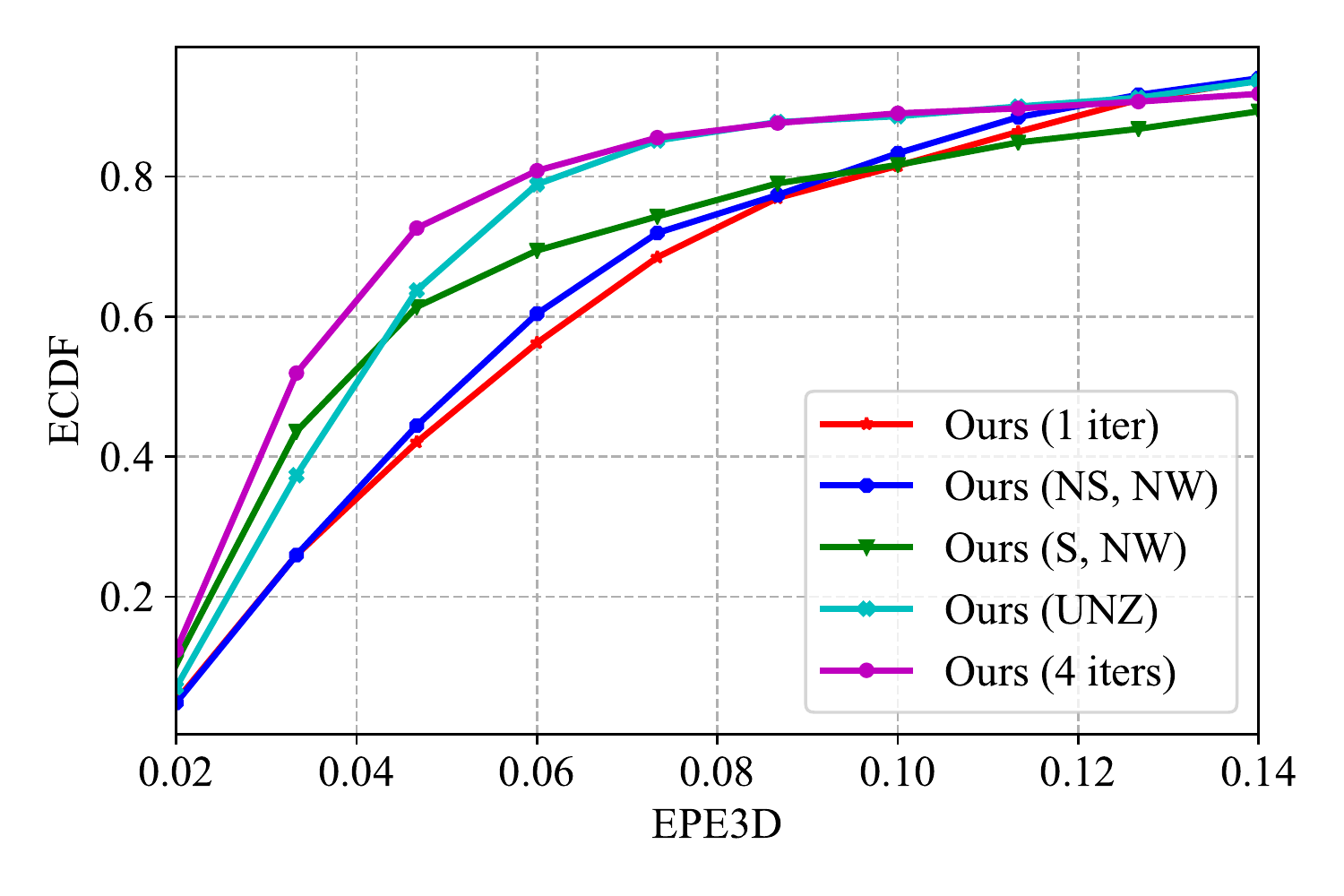}
    \caption{Empirical cumulative distribution function (ECDF) of rigid flow error (EPE3D) on SAPIEN~\cite{Xiang_2020_SAPIEN} dataset. The higher the curve, the better the results.}
    \label{fig:ablation-flow-sapien}
\end{figure}

\begin{figure}[t]
    \centering
    \includegraphics[width=0.8\linewidth]{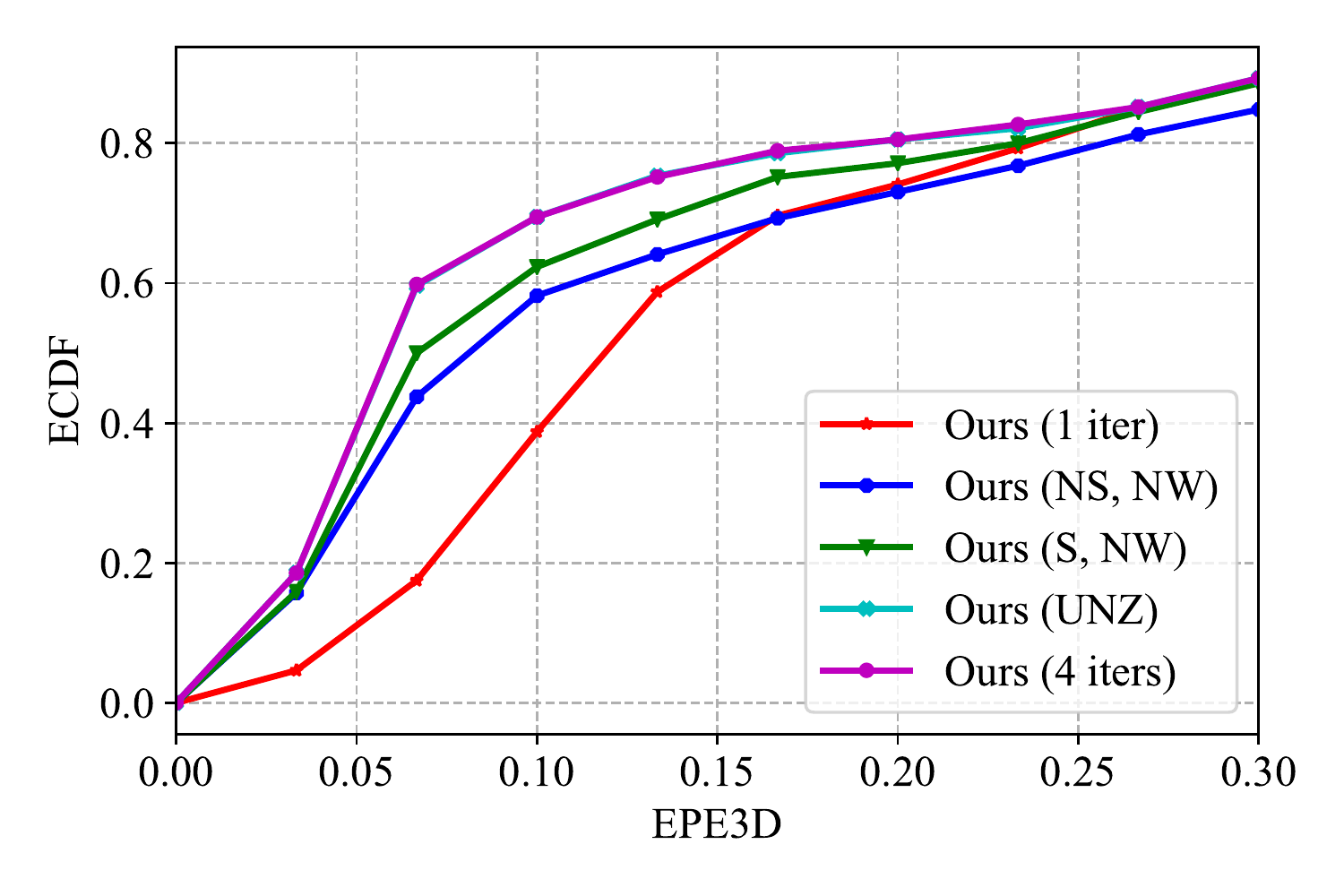}
    \caption{Empirical cumulative distribution function (ECDF) of rigid flow error (EPE3D) on \dataset dataset. The higher the curve, the better the results.}
    \label{fig:ablation-flow-dynlab}
\end{figure}

\begin{figure}[!t]
    \begin{center}
       \includegraphics[width=0.95\linewidth]{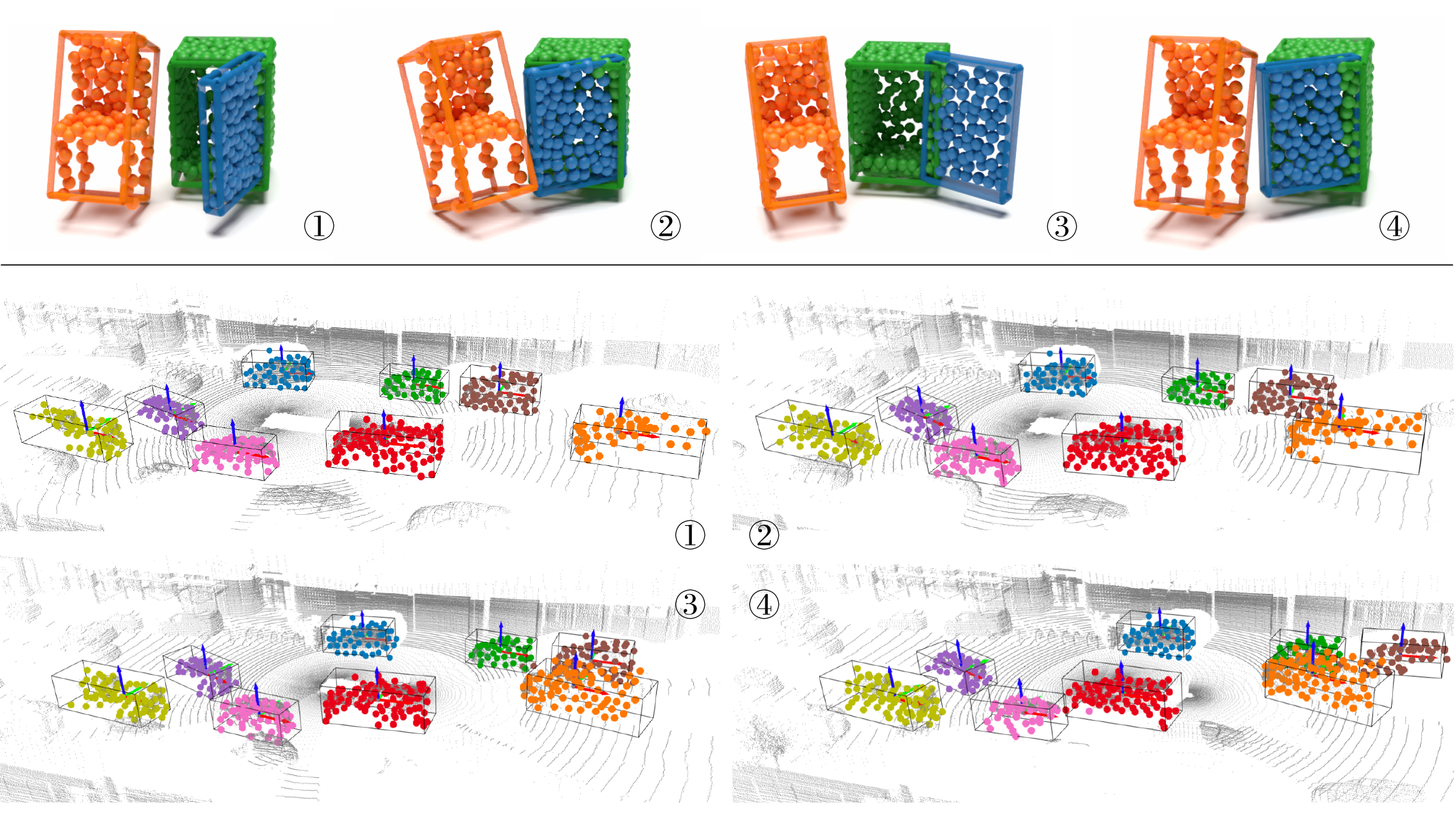}
    \end{center}
   \vspace{-1em}
    \caption{Quantitative demonstrations on complex scans. The first row is estimated using our trained articulated objects model while the last row is obtained by hierarchically apply our method to each segmented part until convergence. \ding{172}-\ding{175} indicates scan index. Best viewed with 200\% zoom in.}
 \label{fig:hier}
 \vspace{-4em}
 \end{figure}

\subsection{Qualitative Results}

To provide the readers with a more intuitive understanding of our performance under different cases, we illustrate in \cref{fig:hier} the scenarios with co-existing articulated/solid objects and multiple cars in a scene of Waymo Open dataset~\cite{sun2020scalability} (though the car category is within our training set).
Moreover, we show in~\cref{fig:res-col0,fig:res-col1,fig:res-col2} our segmentation and registration results for each category in SAPIEN~\cite{Xiang_2020_SAPIEN} dataset, covering most of the articulated objects in real world.
Due to the irregular random point sampling pattern and the natural motion ambiguity, in some examples, our method may generate excessive rigid parts, which can be possibly eliminated by a carefully-designed post-processing step and is out of the scope of this work.
We also show results from the \dataset dataset in~\cref{fig:res-col3}.
Our method can generate robust object associations under challenging settings.

\clearpage

 \begin{figure}
    \centering
    \includegraphics[width=\linewidth]{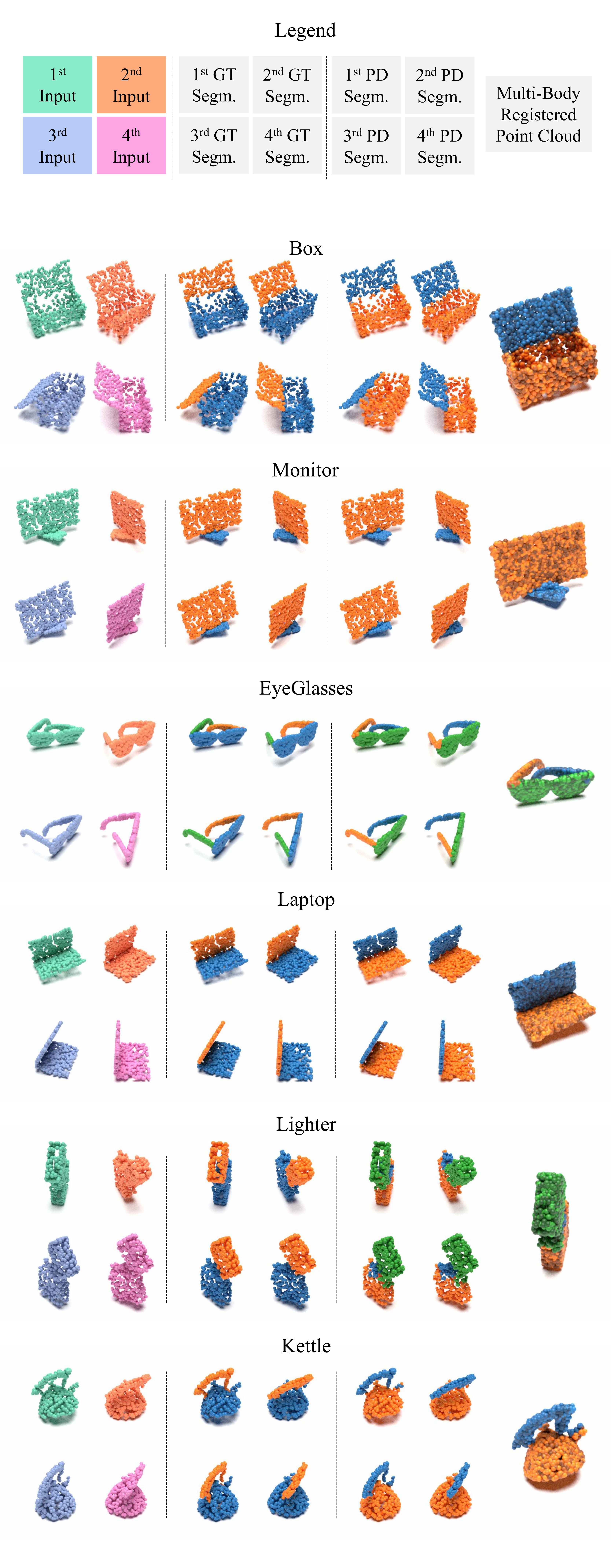}
    \caption{Qualitative results on SAPIEN dataset (1/3).}
    \label{fig:res-col0}
\end{figure}

\begin{figure}
    \centering
    \includegraphics[width=\linewidth]{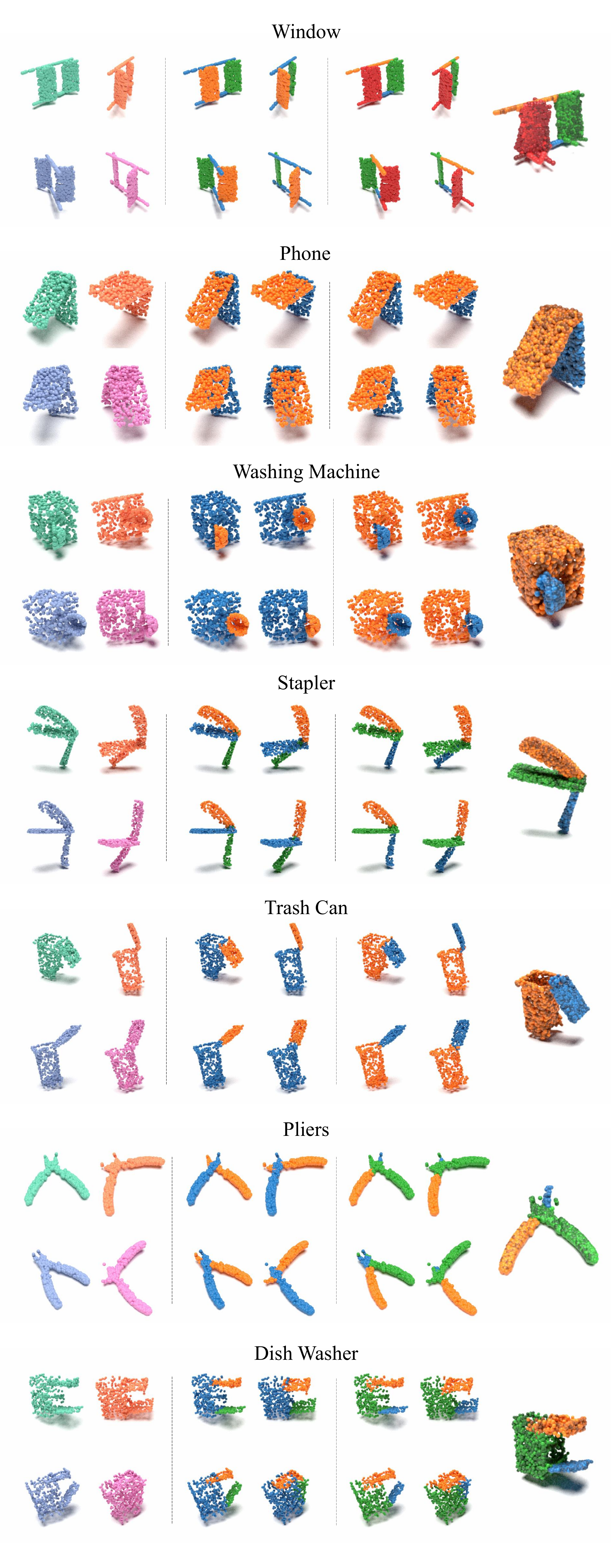}
    \caption{Qualitative results on SAPIEN dataset (2/3).}
    \label{fig:res-col1}
\end{figure}

\begin{figure}
    \centering
    \includegraphics[width=\linewidth]{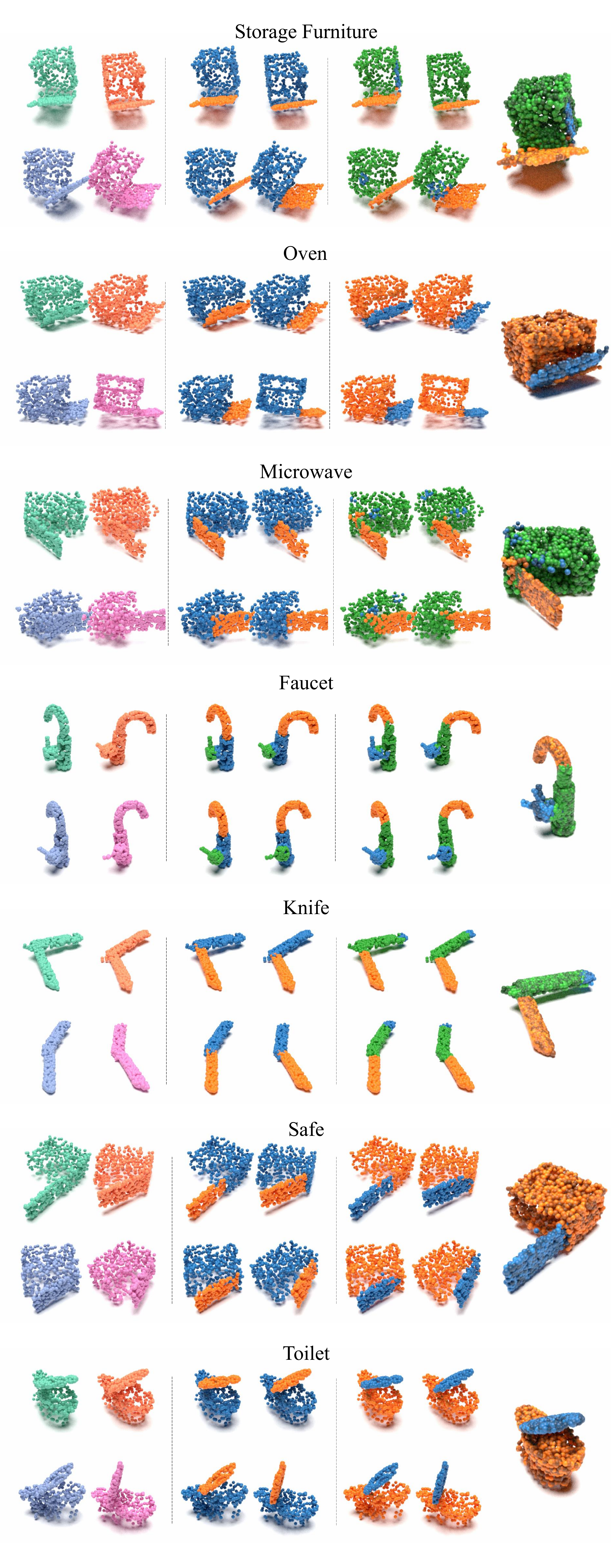}
    \caption{Qualitative results on SAPIEN dataset (3/3).}
    \label{fig:res-col2}
\end{figure}

\begin{figure}
    \centering
    \includegraphics[width=\linewidth]{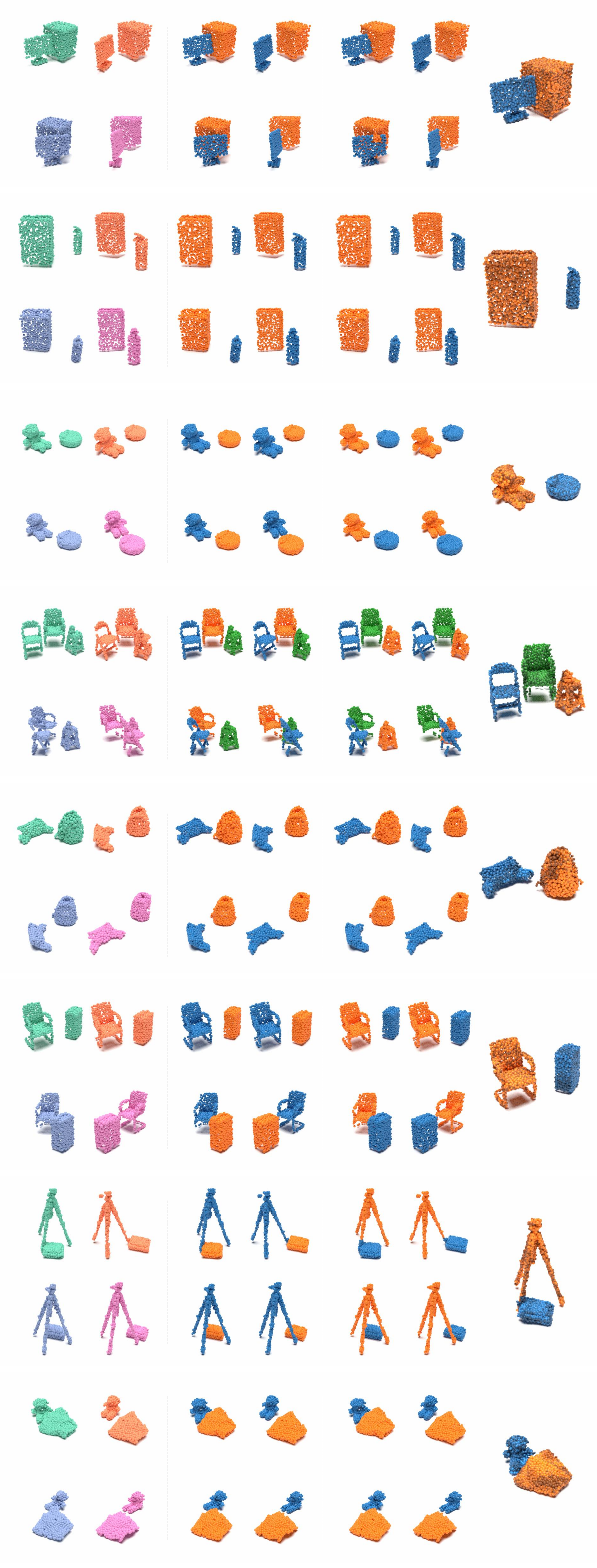}
    \caption{Qualitative results on \dataset dataset.}
    \label{fig:res-col3}
\end{figure}

%% file: table/sapien_category.tex
\begin{table*}
\centering
\caption{Per-category mIoU comparisons on SAPIEN~\cite{Xiang_2020_SAPIEN} dataset.}
\label{tbl:sapien-category}
\small
\setlength{\tabcolsep}{7.0pt}
\begin{tabular}{l|cccccccccc} 
\toprule
                         & Box  & Dishwasher & Display & \begin{tabular}[c]{@{}c@{}}Storage\\Furniture \end{tabular} & Eyeglasses & Faucet & Kettle & Knife & Laptop & Lighter  \\ 
\hline
PointNet++~\cite{qi2017pointnet++}               & 43.5 & 46.8       & 54.8    &  51.3 & 34.6       & 42.4   & 65.7   & 43.0  & 58.5   & 52.3           \\
MeteorNet~\cite{liu2019meteornet}                & 47.0 & 42.2       & 41.7   &  36.9 & 36.1       & 47.1   & 67.2   & 36.2  & 57.9   & 61.3       \\
DeepPart~\cite{yi2018deep}                 & 53.3 & 55.1       & 47.4    &  48.7 & 31.8       & 43.4    & 64.7   & 38.5  & 67.3   & 39.0          \\
NPP~\cite{hayden2020nonparametric}                      & 41.4 & 63.7       & 57.3    &  48.0 & 35.3       & 45.4   & 50.7   & 44.5  & 61.1   & 50.7          \\ 
\hline
Ours (1 iter)            & 67.6 & 57.3       & 66.3    &  68.1 & 57.8       & 54.7   & 83.3   & 55.5  & 78.6   & 52.0       \\
Ours (NS, NW)            & 67.1 & 61.6       & 62.6    &  67.5 & 60.6       & 50.3   & 78.3   & 53.6  & 77.5   & 51.5        \\
Ours (S, NW)             & 71.4 & 58.9       & 68.8    &  71.3 & 61.7       & 57.2   & 81.4   & 57.8  & 82.7   & 64.6        \\
Ours (UNZ)               & 71.5 & 59.6       & 69.1    &  71.6 & 62.1       & 58.0   & 78.9   & 57.8  & 82.7   & 65.0        \\
\textbf{Ours (4 iters) } & 72.0 & 62.0       & 67.4    &  73.1 & 66.2       & 56.2   & 80.7   & 56.4  & 83.3   & 62.6     \\
\bottomrule
\end{tabular}

\vspace{1em}

\setlength{\tabcolsep}{6.5pt}
\begin{tabular}{l|cccccccccc|c} 
\toprule
                          & Microwave & Oven & Phone & Pliers & Safe & Stapler & Door & Toilet & TrashCan & \begin{tabular}[c]{@{}c@{}}Washing\\Machine \end{tabular} & Overall  \\ 
\hline
PointNet++~\cite{qi2017pointnet++}                & 51.5      & 42.6 & 46.2  & 63.6   & 55.7 & 43.0    & 42.7                                                        & 40.0   & 51.2     & 49.8                                                      & 47.5     \\
MeteorNet~\cite{liu2019meteornet}                 & 37.4      & 37.1 & 41.7  & 43.4   & 33.7 & 54.7    & 33.3                                                        & 38.3   & 61.5     & 41.9                                                      & 43.7     \\
DeepPart~\cite{yi2018deep}                  & 65.9      & 49.8 & 41.9  & 32.9   & 57.5 & 47.0    & 38.6                                                        & 39.1   & 65.1     & 59.5                                                      & 49.2     \\
NPP~\cite{hayden2020nonparametric}                       & 56.4      & 39.7 & 48.4  & 61.3   & 55.9 & 45.5    & 40.4                                                        & 31.2   & 51.0     & 48.4                                                      & 48.2     \\ 
\hline
Ours (1 iter)             & 61.6      & 54.7 & 52.5  & 50.6   & 59.4 & 67.0    & 47.1                                                        & 55.7   & 79.3     & 64.2                                                      & 62.9     \\
Ours (NS, NW)             & 74.6      & 59.0 & 49.4  & 57.0   & 62.2 & 65.6    & 45.1                                                        & 52.0   & 76.1     & 72.9                                                      & 63.3     \\
Ours (S, NW)              & 62.8      & 52.3 & 54.1  & 51.4   & 62.5 & 72.0    & 49.0                                                        & 57.2   & 81.1     & 71.2                                                      & 65.6     \\
Ours (UNZ)                & 62.7      & 52.2 & 55.1  & 49.9   & 61.3 & 72.3    & 48.8                                                        & 57.5   & 81.4     & 71.4                                                      & 65.8     \\
\textbf{Ours (4 iters) }  & 69.3      & 56.1 & 54.6  & 63.9   & 63.9 & 70.2    & 48.3                                                        & 56.5   & 80.4     & 72.1                                                      & 66.7     \\
\bottomrule
\end{tabular}

\end{table*}